\documentclass[conference]{IEEEtran}
\IEEEoverridecommandlockouts
\usepackage[T1]{fontenc}

\usepackage{tikz}
\usepackage{pgfplots}
\usepackage{booktabs}
\usepackage{enumerate}
\usepackage{algorithm}
\usepackage{algorithmic}
\usepackage{url}
\usepackage{fontawesome}
\usepackage[misc]{ifsym}
\usepackage{bm}

\usepackage{cite}
\usepackage{amsmath,amssymb,amsfonts,amsthm}
\usepackage{algorithmic}
\usepackage{graphicx}
\usepackage{textcomp}

\usepackage[table]{xcolor}
\usepackage{colortbl}

\definecolor{BuyerBlue}{RGB}{49,130,189}
\definecolor{SellerOrange}{RGB}{255,255,0}
\definecolor{PlatformGreen}{RGB}{49,163,84}
\definecolor{NegRed}{RGB}{203,24,29}

\usepackage{tikz}
\usepackage{pgfplots}
\usepackage{booktabs}
\usepackage{enumerate}
\usepackage{algorithm}
\usepackage{algorithmic}

\usepackage{empheq}

\usepackage[english]{babel}
\usepackage{microtype}       % subtle protrusion/expansion; reduces overfulls
\usepackage{glossaries}
\glsdisablehyper

\newacronym{ml}{ML}{machine learning}
\newacronym{sif}{SIF}{standard interference function}

\newtheorem{definition}{Definition}
\newtheorem{theorem}{Theorem}
\newtheorem{lemma}{Lemma}
\newtheorem{proposition}{Proposition}
\newtheorem{corollary}{Corollary}

\makeatletter
\newcommand{\linebreakand}{%
  \end{@IEEEauthorhalign}
  \hfill\mbox{}\par
  \mbox{}\hfill\begin{@IEEEauthorhalign}
}
\makeatother

\def\BibTeX{{\rm B\kern-.05em{\sc i\kern-.025em b}\kern-.08em
    T\kern-.1667em\lower.7ex\hbox{E}\kern-.125emX}}
\begin{document}

\title{TripleWin: Fixed-Point Equilibrium Pricing for Data-Model Coupled Markets
\thanks{Correspondence to Yun Xiong (yunx@fudan.edu.cn) and Lei You (leiyo@dtu.dk)}
}

\author{
\IEEEauthorblockN{1\textsuperscript{st} Hongrun Ren}
\IEEEauthorblockA{
% \textit{College of Computer Science and Artificial Intelligence} \\
\textit{Fudan University}\\
Shanghai, China \\
renhr20@fudan.edu.cn}
\and
\IEEEauthorblockN{2\textsuperscript{nd} Yun Xiong  \textnormal{\textsuperscript\Letter}}
\IEEEauthorblockA{
% \textit{College of Computer Science and Artificial Intelligence} \\
\textit{Fudan University}\\
Shanghai, China \\
yunx@fudan.edu.cn}
\and
\IEEEauthorblockN{3\textsuperscript{rd} Lei You \textnormal{\textsuperscript\Letter}} 
\IEEEauthorblockA{
% \textit{Department of Engineering Technology} \\
\textit{Technical University of Denmark}\\
Ballerup, Denmark \\
leiyo@dtu.dk} 
  \linebreakand % <------------- \and with a line-break
\IEEEauthorblockN{4\textsuperscript{th} Yingying Wang}
\IEEEauthorblockA{
% \textit{College of Computer Science and Artificial Intelligence} \\
\textit{Fudan University}\\
Shanghai, China \\
22210240301@m.fudan.edu.cn}
\and
\IEEEauthorblockN{5\textsuperscript{th} Haixu Xiong}
\IEEEauthorblockA{
% \textit{College of Computer Science and Artificial Intelligence} \\
\textit{Fudan University}\\
Shanghai, China \\
22110240113@m.fudan.edu.cn}
\and
\IEEEauthorblockN{6\textsuperscript{th} Hao Niu}
\IEEEauthorblockA{
% \textit{College of Computer Science and Artificial Intelligence} \\
\textit{Fudan University}\\
Shanghai, China \\
hniu@fudan.edu.cn}
\and
\IEEEauthorblockN{7\textsuperscript{th} Xin Wang
}
\IEEEauthorblockA{
% \textit{College of Computer Science and Artificial Intelligence} \\
\textit{Tianjin University}\\
Tianjin, China \\
wangx@tju.edu.cn}
% \and
% \IEEEauthorblockN{6\textsuperscript{th} Hao Niu}
% \IEEEauthorblockA{
% % \textit{College of Computer Science and Artificial Intelligence} \\
% \textit{Fudan University}\\
% Shanghai, China \\
% hniu@fudan.edu.cn}
\and
\IEEEauthorblockN{8\textsuperscript{th} Yangyong Zhu}
\IEEEauthorblockA{
% \textit{College of Computer Science and Artificial Intelligence} \\
\textit{Fudan University}\\
Shanghai, China \\
yyzhu@fudan.edu.cn}
}

\maketitle

\begin{abstract}
The rise of the machine learning (ML) model economy has intertwined markets for training datasets and pre-trained models. 
However, most pricing approaches still separate data and model transactions or rely on broker-centric pipelines that favor one side. 
Recent studies of data markets with externalities capture buyer interactions but do not yield a simultaneous and symmetric mechanism across data sellers, model producers, and model buyers. 
We propose a unified \emph{data–model coupled market} that treats dataset and model trading as a single system. 
A supply side mapping transforms dataset payments into buyer visible model quotations, while a demand side mapping propagates buyer prices back to datasets through Shapley-based allocation. 
Together, they form a closed loop that links four interactions: supply–demand propagation in both directions and mutual coupling among buyers and among sellers. 
We prove that the joint operator is a \emph{standard interference function (SIF)}, guaranteeing existence, uniqueness, and global convergence of equilibrium prices. 
Experiments demonstrate efficient convergence and improved fairness compared with broker-centric and one-sided baselines. The code is available on \faGithub~\textit{\footnotesize \url{https://github.com/HongrunRen1109/Triple-Win-Pricing}}.
\end{abstract}


\section{Introduction}
\label{sec:intro}

Training modern \gls{ml} models requires expertise, large scale datasets, and substantial computational resources~\cite{DBLP:journals/jsac/PanSWLLB25,DBLP:journals/pvldb/LiuLL0PS21,DBLP:journals/kais/CongLPZZ22,DBLP:journals/tkde/Pei22}. 
Many organizations therefore outsource model development or purchase pre-trained models from specialized providers, which has fostered a growing model economy and online marketplaces~\cite{modelplace,datarade}. 
Within this economy there have historically been two separate venues. 
A data market connects data sellers with data buyers and studies how datasets are priced and valued~\cite{DBLP:conf/icde/PengMFZDY25,DBLP:journals/tifs/LuHCYJL25,DBLP:conf/ec/AgarwalDS19,DBLP:conf/icml/ChenLX22}. 
A model market connects model producers with model buyers and studies how trained models are transacted~\cite{DBLP:conf/sigmod/ChenK019,DBLP:journals/tdsc/WengWCHW22,DBLP:journals/fi/QianMBGLY25}. 
An important integration effort relies on intermediaries. 
In broker dominated approaches the platform sets prices and redistributes revenue, either by first fixing data side payments and then marking up model prices or by first setting model prices and then allocating revenue back to data sellers~\cite{DBLP:journals/pvldb/LiuLL0PS21,DBLP:conf/infocom/SunCLH22}. 
Parallel to this line, recent theory models fixed price data markets with buyer to buyer externalities and studies platform transaction fees that internalize externalities and deliver equilibria with welfare guarantees~\cite{DBLP:journals/iotj/WangT24}. 
Despite important insights, these approaches do not give a symmetric and simultaneous mechanism that forms prices across the two layers of data and models.

When the two venues are decoupled or when a single intermediary dictates prices, feedback between training costs and market demand is either delayed or imposed unilaterally~\cite{DBLP:conf/cscwd/GaoWGZJYY25}. 
In practice higher dataset payments raise a producer’s training cost and should increase the model price observed by buyers. 
Stronger buyer bids reveal demand and should increase the compensation of high contribution datasets. 
If the feedback does not circulate inside the mechanism, fairness becomes fragile because one side can capture surplus at the expense of the other sides, and efficiency suffers because price signals do not propagate cleanly through the system. 
A simultaneous mechanism is needed, one that treats the two layers as parts of a single market with symmetric roles for data sellers, model producers, and model buyers.

We address this gap by developing a \emph{data–model coupled market} in which training data transactions and model transactions are two interdependent layers of one market. 
Our key modeling decision is to work at the level of per-use payments. 
For each dataset and model pair we introduce a per-use payment that a model pays to a dataset when the dataset is used in training. 
For each model and buyer pair we introduce a buyer specific model price. 
Prices in the two layers are connected through two quotation mappings that formalize how economic information flows through the market. 
A supply side mapping transforms the vector of dataset payments and the producer’s target margin into model quotations seen by buyers. 
A demand side mapping transforms the vector of buyer prices into dataset payments by allocating an effective training time revenue according to Shapley contributions of datasets to models. We also consider producer's margin in the entire market. This design removes the need for a privileged intermediary and lets prices communicate costs and values across the network.

\begin{figure}[t]
  \centering
  \includegraphics[width=\linewidth]{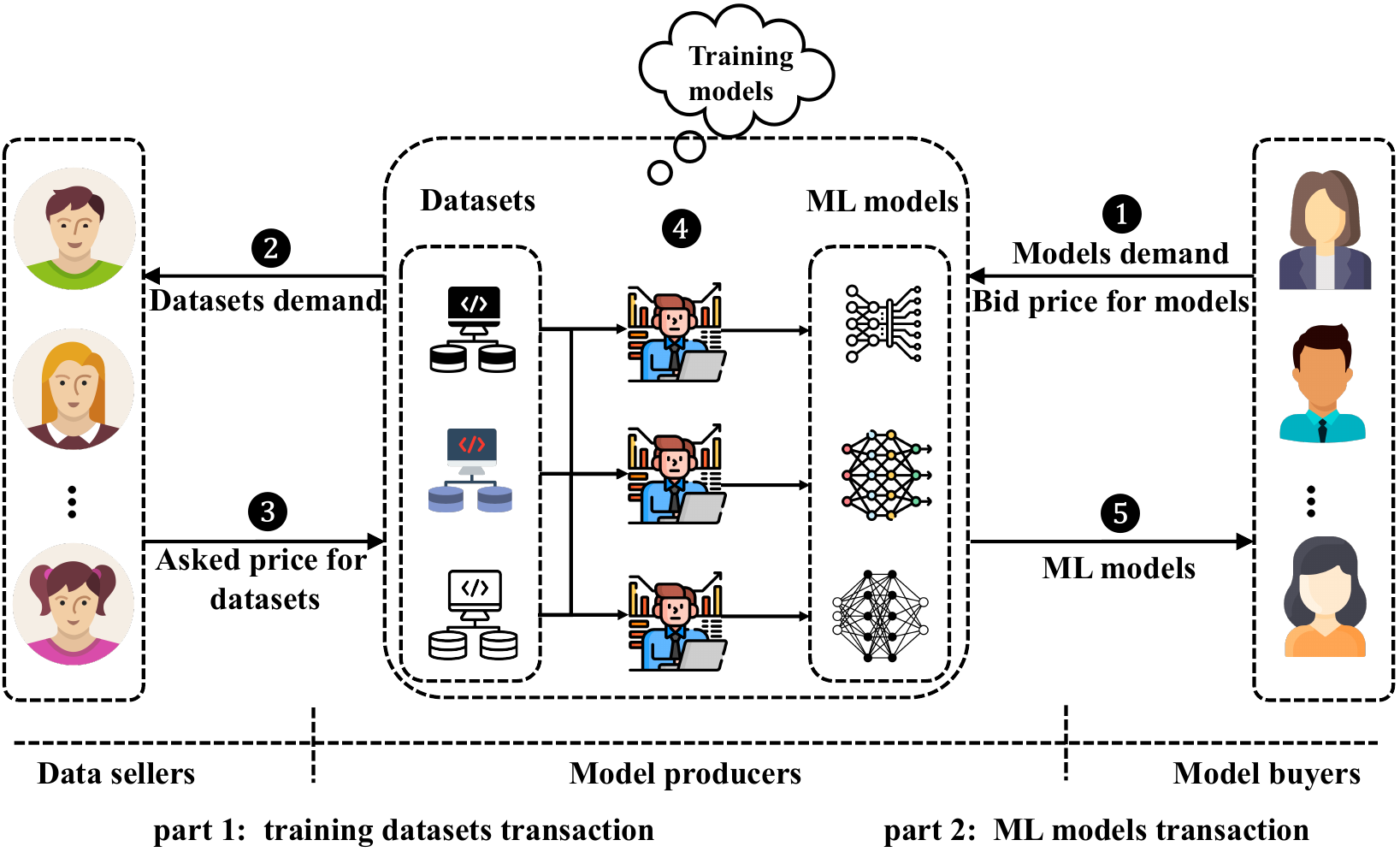} % adjust path if needed
  \caption{Workflow of the data–model coupled market cleared by \textsc{TripleWin}.
  (1) Buyers submit bids prices. 
  (2) Producers select datasets and request licenses. 
  (3) Data sellers post asks prices. 
  (4) Models are trained on licensed data. 
  (5) Trained models are delivered to buyers.
  \textsc{TripleWin} clears dataset licensing and model sales simultaneously by aligning buyer-side and data-side quotations.}
  \label{fig:procedure}
\end{figure}

The central problem we solve is the \emph{triple win pricing problem} in a multi-buyer and multi-seller environment, as illustrated in \figurename~\ref{fig:procedure}.
The goal is to reach prices that the \emph{three parties accept simultaneously}. 
Data sellers must receive per-use compensation consistent with their contribution and their operational costs. 
Model producers (referred to also as ``broker'') must recover training costs and achieve target margins in a sustainable way. 
Model buyers must purchase at prices aligned with their willingness to pay and the realized utility of models. 
We formalize this objective as an alignment between individual posted prices and market quotations. 
When the discrepancy vanishes, posted prices coincide with market quotations and the market clears. 
This alignment is equivalent to a fixed point condition for the joint operator that aggregates the two quotation mappings.

Beyond proposing a simultaneous pricing view, we contribute a theory that guarantees the bidirectional mappings yield a well-defined and computable equilibrium. 
We prove that the joint operator is a \emph{\gls{sif}}~\cite{yates2002framework}. 
Because the mappings are affine with nonnegative coefficients and strictly positive offsets, the operator is positive, monotone, and scalable. 
These properties imply the existence and uniqueness of a fixed point for the simultaneous update of all prices and imply global convergence of a simple fixed-point iteration, including block alternating and asynchronous updates that mirror distributed implementations. 
This foundation provides a principled counterpart to earlier heuristic bargaining loops and complements recent work on platform side fee mechanisms for externality control in data markets~\cite{DBLP:journals/pvldb/ZhuZZLR24,DBLP:journals/tmc/ZhengMXW24,DBLP:journals/ton/SunLCH24}. 
In contrast to broker-centric systems~\cite{DBLP:journals/pvldb/LiuLL0PS21,DBLP:conf/infocom/SunCLH22}, our mechanism is symmetric across the three parties, operating directly with per-use prices at the dataset to model level and with buyer-specific model prices.

The significance is both conceptual and operational. 
Conceptually, we treat data and model pricing as a single, simultaneous process in which price signals propagate across layers in both directions, with buyers influencing one another through revenue aggregation for a model, and sellers influencing one another through the total training expenditure of a model. 
Operationally, the mechanism yields a single self-consistent price vector under broad conditions and supports transparent and repeatable transactions. 
Within the lines of work represented by~\cite{DBLP:journals/tdsc/WengWCHW22,DBLP:conf/infocom/SunCLH22}, a multi-buyer and multi-seller mechanism that operates with per-use payments at the dataset to model level, that closes the price formation loop in both directions, and that comes with a fixed point analysis ensuring a unique equilibrium has not been formalized in this way. 
Our framework therefore complements data(supply)-first and buyer(demand)-first pipelines by offering a unified and theory-driven alternative that is simple to implement yet provides convergence guarantees.

We summarize our contributions as follows.
\begin{itemize}
\item We introduce a data–model coupled market that integrates dataset and model trading into one simultaneous pricing environment with symmetric roles for data sellers, model producers, and model buyers.
\item We define the triple win objective as agreement between posted prices and market quotations and we instantiate bidirectional quotation mappings that propagate price information across the two layers using Shapley-based allocation on the data side.
\item We prove that the joint pricing operator is an SIF and that a unique fixed point exists; the fixed-point iteration converges globally, including in alternating block and asynchronous forms.
\item We implement the mechanism with per-use prices at the dataset to model level and we show experimentally that it converges efficiently and improves transaction success and fairness relative to broker-centric and one-sided baselines.
\end{itemize}

\noindent \emph{Organization.}
Section~\ref{related} introduces the related work. Section~\ref{sec:pricing} formulates the triple win pricing problem and defines the bidirectional mappings. 
Section~\ref{sec:theory} develops the theoretical guarantees for existence, uniqueness, and convergence. 
Section~\ref{sec:experiments} reports the empirical evaluation. 
Section~\ref{sec:conclusion} concludes.

\section{Related Work}
\label{related}

The data trading market, particularly in the context of \gls{ml} and model transactions, has become a hot research topic in recent years. Traditional studies have primarily focused on independent pricing mechanisms for data markets~\cite{DBLP:journals/tifs/ChristensenPP23,DBLP:journals/tkde/MiaoGCPYL22,DBLP:journals/compsec/ShenGSDDZZJ22,DBLP:conf/icde/FuMPNDY24,DBLP:journals/tdsc/LuZR24,DBLP:journals/jcst/LiuZWC24} and model markets~\cite{DBLP:conf/ant/BatainehMBB16,DBLP:conf/aciids/StahlV16,DBLP:journals/pvldb/PeiF023,DBLP:conf/apsipa/HuangCSWCYCL23,DBLP:journals/corr/abs-2504-04794}. For instance, in the realm of data pricing, Wen et al.~\cite{DBLP:conf/kdd/WenFWHX0HWJ25} proposed a privacy-preserving pre-training data pricing mechanism specifically designed to address the data pricing issues in federated learning. This mechanism ensures data privacy while providing appropriate pricing for training data used by models. Bi et al.~\cite{DBLP:conf/icde/BiLZZ0024} applied Stackelberg game theory and Nash equilibrium to propose a new data market pricing model, describing a data market composed of buyers, brokers, and multiple sellers, where buyers purchase data products through queries and sellers provide privacy-protected data, with all participants aiming to maximize their own profits.

In terms of model pricing, Chen et al.~\cite{DBLP:conf/icml/ChenLX22} introduced a model-based pricing framework that formalizes the pricing mechanisms for \gls{ml} models, advancing the field of model pricing research. However, as the trade of data and models becomes increasingly intertwined, more research is exploring how to integrate the pricing mechanisms of both to achieve fairer and more efficient market outcomes. Liu et al.~\cite{DBLP:journals/pvldb/LiuLL0PS21} proposed an end-to-end differential privacy \gls{ml} model marketplace, where intermediaries first determine a set of model prices to maximize revenue, though they do not guarantee avoidance of arbitrage. Subsequently, Shapley values are used to simulate the contribution of data owners to the models and compensate them accordingly. Sun et al.~\cite{DBLP:conf/infocom/SunCLH22} designed a federated learning-based \gls{ml} model marketplace with differential privacy, which first calculates the privacy costs of data owners, optimizes model versions, and then sells the generated model versions to model buyers. However, these markets primarily focus on the interests of intermediaries, neglecting the interests of other market participants. The PIECE mechanism proposed by Lu et al.~\cite{DBLP:journals/tifs/LuHCYJL25} sets model prices through brokers, organizes data owners to participate in training, and generates model versions based on buyer demands and privacy budgets. It ultimately allocates rewards based on the contributions of data owners to maximize market revenue. Pan et al.~\cite{DBLP:journals/jsac/PanSWLLB25} optimized this process by optimizing worker recruitment based on data quality and bidding, ensuring maximum custom model quality while ensuring reasonable cost distribution between the platform and workers.

However, existing integrated research typically encounters a problem: model prices are usually determined first, and then compensation for data sellers is calculated based on those prices~\cite{DBLP:journals/tifs/LuHCYJL25,DBLP:conf/infocom/SunCLH22}; alternatively, data prices are determined first, and model prices are then calculated accordingly. This sequential pricing approach fails to adequately account for the interdependencies and dynamic feedback between data and models~\cite{DBLP:journals/jsac/PanSWLLB25}.

In contrast, \textsc{TripleWin} provides a tri-sided, data–model–buyer clearing mechanism that forms prices simultaneously by coupling buyer-side and data-side quotations, rather than relying on broker-centric or one-sided pipelines. Its pricing operator is an \gls{sif}, ensuring a unique equilibrium and global convergence. Aggregated buyer willingness to pay is translated into dataset compensation via Shapley allocations, producer margins enter transparently, and buyer reserves bound feasibility, so incentives are aligned across all parties. This coupling removes systematic bias, delivers fair incidence of value, and measurably improves transaction success, efficiency, and stability. Practically, it yields a simple, transparent rule that scales to multi-buyer and multi-seller markets with provable guarantees.

\section{Bidirectional TripleWin Pricing} 
\label{sec:pricing}

In this section we define the entities, their interactions, and the notations used in our multi-buyer and multi-seller market.

\begin{table}[t]
\centering
\caption{Notation.}
\label{tab:notation}
\begin{tabular}{ll}
\toprule
Symbol & Meaning \\
\midrule
$D_i$, $M_j$, $B_k$ & Dataset $i$, model (producer) $j$, buyer $B_k$ \\
$\mathcal{D}_{M_j}$, $\mathcal{B}_{M_j}$ & Datasets used by $M_j$; buyers of $M_j$ \\
$p_{D_i\to M_j}$ & Posted price from data seller $i$ to model $j$ \\
$p_{B_k\to M_j}$ & Bid price from buyer $k$ to model $j$ \\
$C_{D_i\to M_j}$ & Seller cap price \\
$R_{B_k\to M_j}$ & Buyer reserve (max willingness to pay) \\
$v_{D_i\to M_j}$, $v_{B_k\to M_j}$ & Data-side / buyer-side quotations \\
$\kappa_{D_i}$, $\kappa_{M_j}$ & Positive offsets (data-side, model-side) \\
$\delta_{M_j}$ & Producer margin for model $j$ \\
$\alpha_{\kappa_D},\alpha_{\kappa_M}$, $\alpha_{\delta}$ & Global scaling of $\kappa_{D_i}$, $\kappa_{M_j}$, and $\delta_{M_j}$ \\
$\mathrm{SV}_{i\mid j}$ & Shapley share of $D_i$ for $M_j$; $\sum_{i\in \mathcal{D}_{M_j}}\!\!\!\!\mathrm{SV}_{i\mid j}\!\!=\! 1$ \\
$\omega_{jk}$, $\rho_j$ & Buyer weights, with $\sum_{k}\omega_{jk}=\rho_j\in[0,1)$ \\
$W_j$ & Effective aggregation $W_j=\sum_k\omega_{jk}p_{B_k\to M_j}$ \\
$\mathcal A_B,\mathcal A_D$ & Acceptance sets: $p_{B\to M}\le R$, $p_{D\to M}\ge \kappa_D$ \\
\bottomrule
\end{tabular}
\end{table}

\subsection{Market Entities and Interactions}

Game-theoretic models are widely used to study strategic behavior in data trading~\cite{DBLP:journals/iet-com/HaddadiG16,DBLP:journals/iotj/XiaoHD21,DBLP:journals/dga/HelmesS15,DBLP:journals/tmc/TianWZDLZ25}. 
In our data--model coupled market, pricing decisions by \emph{data sellers}, \emph{model producers}, and \emph{model buyers} influence one another. 
A triple win outcome occurs when all settled edge prices lie within the three sides' acceptance ranges. 
\figurename~\ref{fig:data-model-pricing} sketches the two quoting directions (demand-driven and supply-driven), the posted/bidding prices, and the propagation/mutual-influence effects across agents.

We collect notation in \tablename~\ref{tab:notation}. 
Each seller controls one dataset \(D_i\). 
Each producer trains one model \(M_j\) using a set of datasets \(\mathcal{D}_{M_j}\subseteq\{D_1,\dots,D_\ell\}\) and sells that model to a set of potential buyers \(\mathcal{B}_{M_j}\).
We adopt per-use licensing with \emph{edge prices}: for every data--model edge \((i,j)\) there is a price \(p_{D_i\to M_j}\ge 0\); for every buyer--model edge \((k,j)\) there is a buyer-specific price \(p_{B_k\to M_j}\ge 0\). Denote $d=\sum_{j}|\mathcal{B}_{M_j}|+\sum_{j}|\mathcal{D}_{M_j}|$ the total number of scalar edge prices provided by buyers and sellers.
Producer \(j\) targets a profit margin \(\delta_{M_j}\ge 0\). The lowest acceptable price of $D_i$'s seller is $\kappa_{D_i}$ for each dataset $D_i$, which denotes the basic cost of storing the dataset $D_i$. Buyer $B_k$'s reservation (budget) for model $M_j$ is denoted by $R_{B_k\to M_j}$. These ranges are not revealed to the other parties. 

\subsection{Model and Data Valuation}

To fold multi-sale revenue back to a single training episode, define the \emph{effective training revenue}
\[
W_j\big(\{p_{B_k\to M_j}\}_{k\in\mathcal{B}_{M_j}}\big)=\sum_{k\in\mathcal{B}_{M_j}}\omega_{jk}\,p_{B_k\to M_j},\ \sum_{k}\omega_{jk}\leq \rho_j,
\]
with $\omega_{jk},\rho_j\in[0,1)$.
This normalization serves three purposes. 
\emph{(1) Economic meaning.} 
$W_j$ is an \emph{expected one-sale price}: if \(\omega_{jk}\) encodes the beforehand probability that the first realized sale is to buyer $B_k$ (or, more generally, the share of one unit of training-time cost recovery attributed to $B_k$), then \(W_j=\mathbb{E}[p_{M_j\rightarrow B}]\).
Thus, \(W_j\) has the same units and scale as a single model price, making it directly comparable to per-use data compensation on the training episode. 
\emph{(2) Identifiability and fairness.} 
Requiring \(\sum_k \omega_{jk}\leq \rho_j < 1\) prevents double counting of revenue across multiple buyers and removes spurious dependence on how the buyer set is partitioned or expanded; without normalization, \(W_j\) would inflate with the number of buyers, mechanically overpaying data and distorting incentives. 
\emph{(3) Compatibility and theory.} 
The convex combination reduces to the unified-price special case (\(\omega_{jk^\star}=\rho_j \Rightarrow W_j=p_{M_j\rightarrow B_{k^\star}}\)). 
In practice, \(\omega_{jk}\) can be chosen from demand forecasts (probabilities of the first sale), service-level quotas, or policy weights for allocating exactly one unit of effective revenue back to the training episode.

 To fairly evaluate the contribution of each dataset in the collection $\mathcal{D}_{M_{j}}$ to the model $M_{j}$, we employ the Shapley value~\cite{DBLP:conf/aistats/JiaDWHHGLZSS19}. The Shapley value is a solution concept in game theory, used to assess the fair distribution of gains or benefits when participants cooperate and contribute unequally~\cite{DBLP:conf/infocom/XuZWC22}. It satisfies fundamental requirements for market fairness, including balance, symmetry, zero element, and additivity. In this paper, the contribution degree of dataset $D_{i}$ for model $M_{j}$ is represented as $\mathrm{SV}_{i\mid j}$.

\begin{equation}
    \label{eq7}
    \mathrm{SV}_{i\mid j}=\frac{1}{\left|\mathcal{D}_{M_{j}}\right|} \sum_{S \subseteq \mathcal{D}_{M_{j}} \backslash D_{i}} \frac{\mathcal{U}\left(S \cup\left\{D_{i}\right\}\right)-\mathcal{U}(S)}{\left(\begin{array}{c}\left|\mathcal{D}_{M_{j}}\right|-1 \\|S|\end{array}\right)}    
\end{equation}
where $\mathcal{U}$ is the utility function, $\mathcal{D}_{M_{j}}$ is the complete collection of datasets for training the model $M_{j}$, and $S\subseteq \mathcal{D}_{M_{j}}$ represents a subset of the collection.

\subsection{Quotation Functions}
\label{subsec:pricing_mappings}

\begin{figure}[t]
    \centering
    \scalebox{0.6}{\input{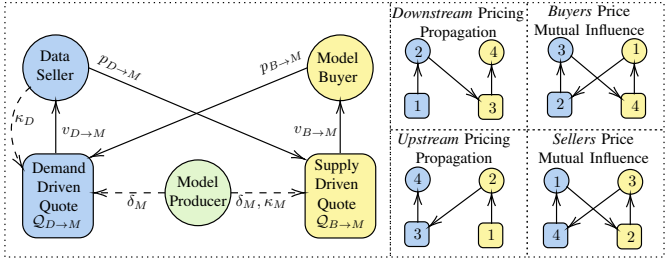}}
\caption{Bidirectional pricing in the coupled data--model market. 
Left: data sellers post $p_{D\to M}$ and buyers bid $p_{B\to M}$; the market maps these to a \emph{supply quote} $\mathcal{Q}_{B\to M}$ and a \emph{demand quote} $\mathcal{Q}_{D\to M}$, moderated by producer margins $\delta_M$ and overheads $(\kappa_D,\kappa_M)$ (dashed arrows indicate internal cost transmission). 
At equilibrium, posted prices match quotes ($p=v$). 
Right: price signals propagate \emph{downstream} (data $\rightarrow$ model $\rightarrow$ buyer) and \emph{upstream} (buyer $\rightarrow$ model $\rightarrow$ data).}
\label{fig:data-model-pricing}
\end{figure}

For model $M_j$ trained on datasets $\mathcal{D}_{M_j}$ and traded with buyers $B_k$,
the posted prices are $p_{D_i\to M_j}$ (data seller $\to$ model producer) and
$p_{B_k\to M_j}$ (buyer $\to$ model producer). The buyer‑side and data‑side quotations are
\begin{subequations}
\begin{empheq}[box=\fbox]{align}
\textbf{Buyer: }v_{B_k\to M_j} &= \kappa_{M_j} + (1+\delta_{M_j})\sum_{i\in \mathcal{D}_{M_j}} p_{D_i\to M_j},
\label{eq:buyer_quote}\\[4pt]
\textbf{Seller: }v_{D_i\to M_j} &= \kappa_{D_i} + \frac{\mathrm{SV}_{i\mid j}}{1+\delta_{M_j}}\, W_j
\label{eq:data_quote}
\end{empheq}
\end{subequations}
\noindent where
\begin{equation}
W_j=\sum_{k\in\mathcal{B}_{M_j}}\omega_{jk}\,p_{B_k\to M_j},\qquad 
\sum_{k\in\mathcal{B}_{M_j}}\omega_{jk}=\rho_j\in[0,1). \label{eq:Q_def}
\end{equation}

Buyers accept if $p_{B_k\to M_j}\le R_{B_k\to M_j}$ and data sellers accept if $p_{D_i\to M_j}\ge \kappa_{D_i}$. Denote $\mathbf p=[\mathbf p_{B\to M};\,\mathbf p_{D\to M}]$.
We evaluate feasibility as $\mathbf p\in\mathcal A_B\cap\mathcal A_D$ with
$\mathcal A_B=\{p_{B\to M}\le R\}$ and $\mathcal A_D=\{p_{D\to M}\ge \kappa_D\}$.

In short, this quotation mechanism propagates individual posted prices upstream from buyers to data sellers and downstream from data sellers to buyers. Two layers of prices are then jointly determined (see \figurename~\ref{fig:data-model-pricing}):
\begin{itemize}
  \item \emph{Downstream propagation:} data expenditures \(\{p_{D_i\to M_j}\}_{i\in \mathcal{D}_{M_j}}\) feed into buyer-specific model prices through the producer's margin, so higher training costs push prices to buyers.
  \item \emph{Upstream propagation:} buyers’ bidding pressure \(\{p_{B_k\to M_j}\}_{k\in \mathcal{B}_{M_j}}\) for \(M_j\) affects the model's overall quote and in turn influences the feasible range of data expenditures.
  \item \emph{Mutual influence within each side:} buyers connected to the same model can affect one another's effective quotes; similarly, sellers of datasets used by the same model interact through the model’s training budget.
\end{itemize}
Crucially, dataset and model prices are solved for \emph{simultaneously}; we avoid pipelines that fix one layer while leaving the other unspecified.

\subsection{TripleWin Pricing Problem Formulation}
\label{subsec:problem_formulation}

The essence of the data–model coupled market is to achieve a triple win outcome in which the pricing expectations of all three parties, data sellers, model producers, and model buyers, are satisfied simultaneously. 
Each party enters the market with its own valuation: data sellers post their asking prices for per-use access to datasets, model producers determine profit margins and cost pass-through to models, and model buyers submit their bids reflecting demand and willingness to pay. 
The market, in turn, responds with quotations that reconcile these competing valuations based on the joint influence of data and model layers. 
A triple win equilibrium occurs when these private prices align with the market’s quotations so that all entities are content with the transaction outcome. 

Mathematically, this interaction can be cast as a least-squares alignment between individual prices and market quotations:
\begin{subequations}
\begin{align}
& \min\ 
\|\mathbf{v}-\mathbf{p}\|_2^2 \label{eq:triplewin_obj} \\[3pt]
\text{s.t.}\quad &
\underbrace{\begin{bmatrix}
\mathbf{v}_{B\to M}\\[4pt]
\mathbf{v}_{D\to M}
\end{bmatrix}}_{\text{quoted prices}}
=
\underbrace{\begin{bmatrix}
\mathcal{Q}_{B\to M}(\mathbf{p}_{D\to M}) \\[4pt]
\mathcal{Q}_{D\to M}(\mathbf{p}_{B\to M})
\end{bmatrix}}_{\text{market mechanism}}.
\label{eq:triplewin_constraint}
\end{align}
\label{eq:problem}
\end{subequations}
The vector $\mathbf{p}$ collects all individual posted prices: $\mathbf{p}_{D\to M}$ from data sellers, representing their requested compensation for each dataset–model pair; $\mathbf{p}_{B\to M}$ from model buyers, representing their bids for each model purchase; and the implicit margins $\delta_{M_j}$ set by model producers that connect the two layers. 
The vector $\mathbf{v}$ denotes the quotations computed by the market itself after processing these inputs through the bidirectional mappings $\mathcal{Q}_{D\to M}$ and $\mathcal{Q}_{B\to M}$. 
The upper mapping $\mathcal{Q}_{B\to M}$ maps data‑side prices to buyer‑facing quotations, while the lower mapping $\mathcal{Q}_{D\to M}$ maps buyer‑side prices to dataset‑facing quotations according to their Shapley-based contribution shares. 
Through these two intertwined processes, information and value flow in opposite directions, linking supply, production, and demand into a single valuation system.

The objective in~\eqref{eq:triplewin_obj} quantifies total market inconsistency: it measures how far each entity’s posted price deviates from the system-level quotation implied by everyone’s actions. 
Minimizing this discrepancy is equivalent to driving the market toward mutual acceptance, where no participant perceives a mismatch between individual valuation and collective outcome. 
When the minimum is attained at zero, all prices satisfy $\mathbf{v}=\mathbf{p}$, implying that individual expectations coincide with the market quotations. 
At that point, the constraint~\eqref{eq:triplewin_constraint} reduces to the fixed-point condition $\mathbf{p}=\mathcal{Q}(\mathbf{p})$. This formulation therefore expresses the triple win condition as an optimization problem: the market seeks a set of prices that jointly minimize tension among sellers, producers, and buyers.

\subsection{TripleWin Pricing Algorithm}

\begin{figure}[t]
    \centering
    \includegraphics[width=0.7\linewidth]{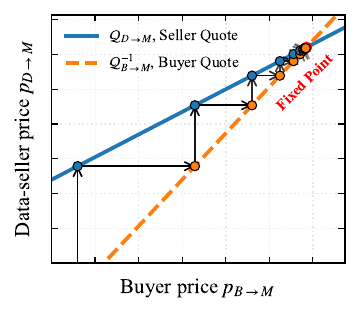}
\caption{Cobweb diagram for the \textsc{TripleWin} fixed-point update in the $(x,y)$ plane with $x=p_{B\to M}$ (buyer price) and $y=p_{D\to M}$ (seller price). The solid curve is the data-side quotation $y=\mathcal{Q}_{D\to M}(x)$. The dashed curve is the buyer-side quotation drawn as a locus in the plane, $x=\mathcal{Q}_{B\to M}(y)$. Starting from an initial $x_0$ on the horizontal axis, the iteration alternates the vertical update $y_{t+1}=\mathcal{Q}_{D\to M}(x_t)$ and the horizontal update $x_{t+1}=\mathcal{Q}_{B\to M}(y_{t+1})$, producing the staircase arrows that converge to $(x^\star,y^\star)$, where $x^\star=\mathcal{Q}_{B\to M}(y^\star)$ and $y^\star=\mathcal{Q}_{D\to M}(x^\star)$.}
    \label{fig:triplewin-fixed-point}
\end{figure}

The procedure in Algorithm~\ref{alg:triplewin} implements \textsc{TripleWin}. At iteration \(t\) it first aggregates buyer bids into an effective quantity \(W_j^{(t)}=\sum_{k\in\mathcal B_{M_j}}\omega_{jk}\,p^{(t)}_{B_k\to M_j}\) for each model \(M_j\) with \(\sum_k\omega_{jk}=\rho_j\in[0,1)\). It then computes the buyer‑side quotation $v^{(t)}_{B_k\to M_j}$
which passes the current total data expenditure and the producer margin to every buyer of model \(M_j\). In parallel it computes the data‑side quotation $v^{(t)}_{D_i\to M_j}$
which allocates the aggregated buyer signal back to each training dataset according to its Shapley share. The state is updated by setting the posted prices to the quotations on both sides, \(\mathbf p^{(t+1)}\gets[\mathbf v^{(t)}_{B\to M};\,\mathbf v^{(t)}_{D\to M}]\). The iteration stops when the normalized fixed‑point residual \(\|\mathcal Q(\mathbf p^{(t)})-\mathbf p^{(t)}\|_2/\sqrt{d}\) is below the tolerance \(\varepsilon\). Because the mapping has strictly positive offsets and nonnegative coefficients, it is an \gls{sif} (established formally in Section~\ref{sec:theory}); existence and uniqueness of the fixed point and global convergence from any nonnegative initialization follow. The offsets guarantee that all iterates remain strictly positive, and the normalization \(\sum_{i\in \mathcal{D}_{M_j}}\mathrm{SV}_{i\mid j}=1\) ensures that quotations are well scaled on the data side.

\begin{algorithm}[htbp]
\caption{\textsc{TripleWin Pricing}}
\label{alg:triplewin}
\begin{algorithmic}
\REQUIRE $\{\mathrm{SV}_{i\mid j}\}$ with $\sum_{i\in \mathcal{D}_{M_j}}\mathrm{SV}_{i\mid j}=1$;
  $\{\omega_{jk}\}$ with $\sum_{k\in\mathcal B_{M_j}}\omega_{jk}=\rho_j\in[0,1)$;
  offsets $\{\kappa_{D_i}\},\{\kappa_{M_j}\}>0$; margins $\{\delta_{M_j}\}\ge 0$;
  initialization $\mathbf p^{(0)}\ge 0$; tolerance $\varepsilon$.
\ENSURE $\mathbf{p}^\star$
\FOR{$t=0,1,2,\ldots$}
  \STATE Compute $W_j^{(t)}=\sum_{k}\omega_{jk}p^{(t)}_{B_k\to M_j}$ for all $j$.
  \STATE Buyer-side quotation:
     $v^{(t)}_{B_k\to M_j}=\kappa_{M_j}+(1+\delta_{M_j})\sum_{i\in \mathcal{D}_{M_j}}p^{(t)}_{D_i\to M_j}$.
  \STATE Data-side quotation:
     $v^{(t)}_{D_i\to M_j}=\kappa_{D_i}+\frac{\mathrm{SV}_{i\mid j}}{1+\delta_{M_j}}\,W_j^{(t)}$.
  \STATE Update prices: $\mathbf p^{(t+1)}\gets [\mathbf v^{(t)}_{B\to M};\,\mathbf v^{(t)}_{D\to M}]$.
  \STATE Stop if $\|\mathcal Q(\mathbf p^{(t)})-\mathbf p^{(t)}\|_2/\sqrt{d}\le \varepsilon$.
\ENDFOR
\RETURN $\mathbf p^\star=\mathbf p^{(t+1)}$.
\end{algorithmic}
\end{algorithm}

The algorithm is lightweight and amenable to vectorization and parallelization. One iteration requires only model‑wise reductions: computing \(W_j^{(t)}\) costs \(\mathcal O(\sum_j|\mathcal{B}_{M_j}|)\), computing \(\sum_{i\in \mathcal{D}_{M_j}}p^{(t)}_{D_i\to M_j}\) costs \(\mathcal O(\sum_j|\mathcal{D}_{M_j}|)\), and forming the quotations is linear in the number of edges. Memory is linear in the number of prices and weights. In practice any nonnegative initialization works; common choices are \(p^{(0)}_{D_i\to M_j}=\kappa_{D_i}\) or policy caps on data edges and \(p^{(0)}_{B_k\to M_j}=\kappa_{M_j}\) on buyer edges. The iteration can be run synchronously as written, or in block‑alternating and asynchronous forms that update a subset of models or edges at a time; the SIF property implies that such variants also converge to the same fixed point provided the updates are fair. Feasibility with respect to buyer reserves and dataset floors is evaluated after convergence; a constrained variant that clips to \(p_{B_k\to M_j}\le R_{B_k\to M_j}\) and \(p_{D_i\to M_j}\ge \kappa_{D_i}\) can be implemented as a post‑update policy upon need. The parameters \(\rho_j\), \(\kappa_{D_i}\), \(\kappa_{M_j}\), and \(\delta_{M_j}\) act as interpretable economic knobs; their comparative‑statics effects are monotone by construction, which makes the iteration robust and its outcomes easy to diagnose. See \figurename~\ref{fig:triplewin-fixed-point}
for a simple illustration of how the algorithm works for a scenario of one buyer and one seller.

\section{Theoretical Aspects}
\label{sec:theory}

All proofs are written for the price vector $ \mathbf{p}=[\mathbf{p}_{B\to M},\,\mathbf{p}_{D\to M}]$
with quoted prices defined through the quotation functions in Section~\ref{subsec:pricing_mappings}. Throughout this section the following standing conditions hold: for every model $M_j$ the profit rate $\delta_{M_j}\ge 0$; for every model $M_j$ $\kappa_{M_j}>0$ and for every dataset $D_i$ $\kappa_{D_i}>0$; for every model $j$ the Shapley weights satisfy $\mathrm{SV}_{i\mid j}\ge 0$ and $\sum_{i\in \mathcal{D}_{M_j}}\mathrm{SV}_{i\mid j}=1$; for every model $j$ the revenue weights satisfy $\omega_{jk}\ge 0$ and $\sum_{k\in\mathcal{B}_{M_j}}\omega_{jk}=\rho_j$. Let 
\begin{equation}
\mathcal{Q}(\mathbf{p}) = \begin{bmatrix}
\mathcal{Q}_{B\to M}(\mathbf{p}_{D\to M}) \\[4pt]
\mathcal{Q}_{D\to M}(\mathbf{p}_{B\to M})
\end{bmatrix}.
\label{eq:I_p}
\end{equation}

\subsection{SIF Properties}
\begin{definition}[SIF]
A mapping $T:\mathbb{R}_{+}^{d}\to \mathbb{R}_{+}^{d}$ is called an SIF if it satisfies the following three properties for all $\mathbf{x},\mathbf{y}\in\mathbb{R}_{+}^{d}$ and all $\alpha>1$:
\begin{itemize}
    \item \emph{Positivity}: $T(\mathbf{x})\in\mathbb{R}_{++}^{d}$;
    \item \emph{Monotonicity}: $\mathbf{x}\ge \mathbf{y}$ implies $T(\mathbf{x})\ge T(\mathbf{y})$;
    \item \emph{Scalability}: $\alpha T(\mathbf{x})>T(\alpha\mathbf{x})$.
\end{itemize}
\end{definition}

\begin{theorem}[Main Result]\label{thm:sif}
The mapping $\mathcal Q$ defined by \eqref{eq:I_p}
is an \gls{sif}: it is positive, monotone, and scalable.
Consequently, if there exists $\bar{\mathbf p}\ge \mathcal Q(\bar{\mathbf p})$,
then the fixed point $\mathbf p^\star=\mathcal Q(\mathbf p^\star)$ exists,
is unique, and the iteration in Algorithm~\ref{alg:triplewin} converges to $\mathbf p^\star$ from any $\mathbf p^{(0)}\ge 0$.
\end{theorem}
\begin{proof}
Positivity holds because all offsets $\kappa_{M_j}$ and $\kappa_{D_i}$ are strictly positive and all coefficients multiplying the prices are nonnegative. Therefore every component of $\mathcal{Q}(\mathbf{p})$ is strictly positive for any $\mathbf{p}\in\mathbb{R}_{+}^{d}$.

Monotonicity holds because each component of $\mathcal{Q}$ is an affine function with nonnegative coefficients of the relevant coordinates of $\mathbf{p}$. If $\mathbf{x}\ge \mathbf{y}$, then for every $j$ and $k$ we have
\begin{multline*}
\kappa_{M_j} + (1+\delta_{M_j})\sum_{i\in \mathcal{D}_{M_j}} x_{D_i\to M_j} \\
\ \ge\ 
\kappa_{M_j} + (1+\delta_{M_j})\sum_{i\in \mathcal{D}_{M_j}} y_{D_i\to M_j},
\end{multline*}
and for every $i$ and $j$ we have
\begin{multline*}
\kappa_{D_i} + \frac{\mathrm{SV}_{i\mid j}}{1+\delta_{M_j}} \sum_{k\in\mathcal{B}_{M_j}}\omega_{jk}\, x_{B_k\to M_j} \\
\ \ge\ 
\kappa_{D_i} + \frac{\mathrm{SV}_{i\mid j}}{1+\delta_{M_j}} \sum_{k\in\mathcal{B}_{M_j}}\omega_{jk}\, y_{B_k\to M_j}.
\end{multline*}
Hence $\mathcal{Q}(\mathbf{x})\ge \mathcal{Q}(\mathbf{y})$.

Scalability holds because the operator is affine with strictly positive offsets. Fix $\alpha>1$ and any $\mathbf{p}$. For every $j$ and $k$,
\[
\alpha \big(\mathcal{Q}(\mathbf{p})\big)_{B_k\to M_j} - \big(\mathcal{Q}(\alpha \mathbf{p})\big)_{B_k\to M_j}
= (\alpha-1)\kappa_{M_j} > 0.
\]
For every $i$ and $j$,
\[
\alpha \big(\mathcal{Q}(\mathbf{p})\big)_{D_i\to M_j} - \big(\mathcal{Q}(\alpha \mathbf{p})\big)_{D_i\to M_j}
= (\alpha-1)\kappa_{D_i} > 0,
\]
since $W_j$ is linear in $\{p_{B_k\to M_j}\}$ and the factor $(1+\delta_{M_j})^{-1}$ is constant. Therefore $\alpha \mathcal{Q}(\mathbf{p})>\mathcal{Q}(\alpha \mathbf{p})$ componentwise and scalability holds. All three properties are satisfied, which proves that $\mathcal{Q}$ is a standard mapping.
\end{proof}

Once the joint operator $\mathcal{Q}$ has been established as an SIF, its structural properties directly entail the main fixed-point results of the classical framework by Yates in \cite{yates2002framework}. Specifically, positivity, monotonicity, and scalability together guarantee that a feasible standard mapping admits a unique and globally attractive fixed point. We restate these consequences in the present context for completeness.

We analyze the fixed point of the simultaneous update \eqref{eq:I_p}. The following feasibility lemma is used.

\begin{theorem}[Existence of Fixed Point Equilibrium Price]
\label{thm:feas}
With the standing conditions $0\leq \rho_j<1$, $\kappa_{M_j}>0$, $\kappa_{D_i}>0$, $\delta_{M_j}\ge 0$, and $\sum_{i\in \mathcal{D}_{M_j}}\mathrm{SV}_{i\mid j}=1$, it is guaranteed that the joint quotation operator $\mathcal{Q}$ admits a unique finite fixed point.
\end{theorem}
\begin{proof}
Fix a model $M_j$ and denote
\[
S_j\;\triangleq\;\sum_{i\in \mathcal{D}_{M_j}}\kappa_{D_i},~
\kappa_{M_j}^{\max}\;\triangleq\;\max_{k\in\mathcal{B}_j}\kappa_{M_j\to B_k}
\]
Define a constant
\[
L_j \;\triangleq\; \frac{\kappa_{M_j}^{\max} + (1+\delta_{M_j})\,S_j}{1-\rho_j}\; \in (0,\infty),
\]
which is finite because $\rho_j<1$ and $\kappa_{M_j}>0$, $\kappa_{D_i}>0$, $\delta_{M_j}\ge 0$.
Construct a vector $\bar{\mathbf{p}}$ by setting, for each $j$,
\[
\bar p_{B_k\to M_j}\;=\;L_j\quad\text{for all }k\in\mathcal{B}_j,
\]
and
\[
\bar p_{D_i\to M_j}\;=\;\kappa_{D_i}+\frac{\mathrm{SV}_{i\mid j}}{1+\delta_{M_j}}\;\rho_j L_j
\quad\text{for all } i\in \mathcal{D}_{M_j}.
\]
Let us verify that $\bar{\mathbf{p}}\ge \mathcal{Q}(\bar{\mathbf{p}})$ componentwise.
First, the seller-side quote at $(i,j)$ is
\[
\mathcal{Q}_{D_i\to M_j}(\bar{\mathbf{p}}_{B\to M})
= \kappa_{D_i} + \frac{\mathrm{SV}_{i\mid j}}{1+\delta_{M_j}}\,Q_j,
\]
with $Q_j=\sum_{k}\omega_{jk}\,\bar p_{B_k\to M_j}=\rho_j L_j$, 
hence $\mathcal{Q}_{D_i\to M_j}(\bar{\mathbf{p}})=\bar p_{D_i\to M_j}$ by construction.

Second, the buyer-side quote at $(k,j)$ is
\begin{multline*}
\mathcal{Q}_{B_k\to M_j}(\bar{\mathbf{p}}_{D\to M})
= \kappa_{M_j\to B_k} + (1+\delta_{M_j})\sum_{i\in \mathcal{D}_{M_j}}\bar p_{D_i\to M_j} \\
= \kappa_{M_j\to B_k} + (1+\delta_{M_j})\Big(S_j+\frac{\rho_j L_j}{1+\delta_{M_j}}\Big) \\
= \kappa_{M_j\to B_k} + (1+\delta_{M_j})S_j + \rho_j L_j.
\end{multline*}
By the choice of $L_j$,
\begin{multline*}
\kappa_{M_j\to B_k} + (1+\delta_{M_j})S_j + \rho_j L_j
\;\le\; \kappa_{M_j}^{\max} + (1+\delta_{M_j})S_j + \rho_j L_j \\
\;=\; (1-\rho_j)L_j + \rho_j L_j \;=\; L_j
\;=\; \bar p_{B_k\to M_j}.
\end{multline*}
Therefore $\mathcal{Q}(\bar{\mathbf{p}})\le \bar{\mathbf{p}}$ componentwise. 

Since $\mathcal{Q}$ is a SIF and there exists $\bar{\mathbf{p}}\in\mathbb{R}_+^d$ with $\bar{\mathbf{p}}\ge \mathcal{Q}(\bar{\mathbf{p}})$, Theorem~1 of Yates~\cite{yates2002framework} (fixed-point theorem for standard interference mappings) implies that $\mathcal{Q}$ admits a \emph{unique} fixed point $\mathbf{p}^\star\in\mathbb{R}_+^d$, and that the iteration $\mathbf{p}^{(t+1)}=\mathcal{Q}(\mathbf{p}^{(t)})$ converges to $\mathbf{p}^\star$ from any nonnegative initialization. Moreover, starting the iteration at $\bar{\mathbf{p}}$ yields a monotone decreasing sequence bounded below by $\mathbf{0}$, so $\mathbf{p}^\star\le \bar{\mathbf{p}}$ and is therefore finite. This proves the lemma.
\end{proof}

\begin{corollary}[Existence of a fixed point]
\label{thm:exist}
Under Theorem~\ref{thm:feas} the operator $\mathcal{Q}$ admits at least one fixed point $\mathbf{p}^{\star}\in\mathbb{R}_{+}^{d}$ with $\mathbf{p}^{\star}=\mathcal{Q}(\mathbf{p}^{\star})$.
\end{corollary}

\begin{corollary}[Uniqueness of the fixed point]
\label{thm:unique}
Under Theorem~\ref{thm:feas} the fixed point is unique.
\end{corollary}

\begin{corollary}[Optimal pricing]
Iterating this block system corresponds to the market’s bidirectional price adjustment
\begin{equation}
\mathbf{p}^{(t+1)}=\mathcal{Q}(\mathbf{p}^{(t)}).
\label{eq:I_iter}
\end{equation}
leads to the fixed point $\mathbf{p}^\star=\mathcal{Q}(\mathbf{p}^\star)$ with $t\to \infty$, which achieves the global optimality of \eqref{eq:problem}.
\end{corollary}

These results follow directly from Theorem~1 in~\cite{yates2002framework}, since our operator $\mathcal{Q}$ satisfies the three defining properties of a SIF and the feasibility condition of Theorem~\ref{thm:feas}. Consequently, the iterative update $\mathbf{p}^{(t+1)}=\mathcal{Q}(\mathbf{p}^{(t)})$ converges from any nonnegative initialization to the unique equilibrium point $\mathbf{p}^{\star}$. This equilibrium represents the market state in which all data, model, and buyer prices are mutually consistent and no participant has an incentive to adjust its pricing decision further.

\subsection{Parameter Monotonicity}

We record a comparative statics statement that is useful in practice. It formalizes the effect of increasing the fixed overheads on the equilibrium prices.

\begin{proposition}[Monotonicity in the offsets]
\label{thm:kappa}
Consider two joint operators $\mathcal{Q}$ and $\widetilde{\mathcal{Q}}$ that only differ in the offsets, where for every edge $\widetilde{\kappa}_{M_j}\ge \kappa_{M_j}$ and for every dataset $\widetilde{\kappa}_{D_i}\ge \kappa_{D_i}$. Assume both operators satisfy Theorem~\ref{thm:feas}. Let $\mathbf{p}^{\star}$ and $\widetilde{\mathbf{p}}^{\star}$ be their fixed points. Then $\widetilde{\mathbf{p}}^{\star}\ge \mathbf{p}^{\star}$ componentwise. If at least one inequality on the offsets is strict, then the corresponding components of the fixed point are strictly larger.
\end{proposition}

\begin{proof}
For every $\mathbf{p}$, the definition of the operators gives $\widetilde{\mathcal{Q}}(\mathbf{p})\ge \mathcal{Q}(\mathbf{p})$ componentwise, with strict inequality on the components linked to the edges or datasets whose offsets are increased. Let $\bar{\mathbf{p}}$ be a common feasible upper bound for both operators. Start from $\bar{\mathbf{p}}$ and iterate both operators. Monotonicity implies the inequality is preserved at every step, hence $\widetilde{\mathbf{p}}^{(t)}\ge \mathbf{p}^{(t)}$ for all $t$. Passing to the limit gives $\widetilde{\mathbf{p}}^{\star}\ge \mathbf{p}^{\star}$. If some offset is strictly increased, the strict part of the inequality propagates through the monotone iteration and yields strict increase in the corresponding components at the limit.
\end{proof}

The monotonicity result shows that increasing any fixed offset $\kappa_{M_j}$ or $\kappa_{D_i}$ raises all equilibrium prices componentwise. Economically, these offsets represent unavoidable per-channel overheads—such as model deployment costs or data-cleaning and compliance costs—so a positive shock to them shifts prices upward throughout the market. The effect is general: higher offsets push costs forward to buyers and backward to data sellers, weakening buyer surplus while strengthening data-side revenue. Producer margins remain unchanged because they scale multiplicatively with $\delta_{M_j}$.

For feasibility, larger offsets shrink the intersection of acceptable price ranges. When buyer reserves or data caps are tight, this can reduce transaction success unless the producer lowers $\delta_{M_j}$ or prunes low-value datasets. For mechanism and policy design, offsets are first-order levers: higher compliance or service guarantees that raise $\kappa_{D_i}$ or $\kappa_{M_j}$ predictably lift all equilibrium prices. Maintaining feasibility then requires adjusting $\omega_{jk}$, moderating margins, or tightening dataset selection.

Although the theorem does not quantify pass-through elasticities, it guarantees directionality: any increase in technological or institutional overhead yields a new equilibrium with weakly higher prices on all edges, shaped by Shapley contributions on the data side and buyer aggregation weights on the demand side.

\subsection{Bounds at the Fixed Point}

The fixed point admits explicit componentwise lower and upper bounds that follow directly from the operator definition and feasibility.

\begin{proposition}[Price Floors and Feasible Ceiling]
\label{prop:pricefloors}
Let $\mathbf{p}^{\star}$ be the unique fixed point. Then for every $j$ and $k$,
\[
p_{B_k\to M_j}^{\star}\ge \kappa_{M_j},
\]
and for every $i$ and $j$,
\[
p_{D_i\to M_j}^{\star}\ge \kappa_{D_i}.
\]
Moreover, if $\bar{\mathbf{p}}$ satisfies $\bar{\mathbf{p}}\ge \mathcal{Q}(\bar{\mathbf{p}})$, then $\mathbf{p}^{\star}\le \bar{\mathbf{p}}$.
\end{proposition}

\begin{proof}
At a fixed point we have $p_{B_k\to M_j}^{\star}=\kappa_{M_j}+(1+\delta_{M_j})\sum_{i\in \mathcal{D}_{M_j}}p_{D_i\to M_j}^{\star}$, which implies $p_{B_k\to M_j}^{\star}\ge \kappa_{M_j}$. Also, 
$p_{D_i\to M_j}^{\star}=\kappa_{D_i}+\frac{\mathrm{SV}_{i\mid j}}{1+\delta_{M_j}}W_j(\{p_{B_k\to M_j}^{\star}\})$ implies $p_{D_i\to M_j}^{\star}\ge \kappa_{D_i}$. The upper bound follows from the monotone decreasing sequence constructed in the proof of Theorem~\ref{thm:exist}, which starts from $\bar{\mathbf{p}}$ and converges to the fixed point. Since every iterate is below $\bar{\mathbf{p}}$, the limit is also below $\bar{\mathbf{p}}$.
\end{proof}

Proposition~\ref{prop:pricefloors} shows that the equilibrium lies between explicit lower and upper bounds. The lower bounds
$p_{B_k\to M_j}^{\star}\ge\kappa_{M_j}$ and $p_{D_i\to M_j}^{\star}\ge\kappa_{D_i}$ have a clear economic meaning: the offsets are irreducible per-channel overheads—deployment, serving, data preparation, or compliance—entering the quotation functions additively and positively. No competition can drive prices below these offsets, which ensures non-degenerate compensation, rules out races to the bottom, and provides a stable floor for revenue planning.

The upper bound $\mathbf{p}^{\star}\le\bar{\mathbf{p}}$ follows from feasibility. A feasible $\bar{\mathbf{p}}\ge\mathcal{Q}(\bar{\mathbf{p}})$ can be derived from buyer reserves, seller caps, or platform policies limiting posted prices. Monotonicity of $\mathcal{Q}$ guarantees that these caps propagate through the mappings, bounding the equilibrium from above. Thus, even under shocks to demand or cost, equilibrium prices cannot exceed what buyers are willing to pay or what policy permits.

Together, the bounds $\kappa\le\mathbf{p}^{\star}\le\bar{\mathbf{p}}$ define a corridor determined by technological and institutional primitives. This corridor provides predictable incidence, stable convergence of fixed-point iterations, and clear design levers: raising offsets lifts the floor, tightening caps or lowering reserves reduces the ceiling, and both adjustments reshape the attainable triple-win region.

\subsection{Feasibility Envelopes}\label{subsec:feasibility-envelopes}

This subsection derives analytic \emph{outer envelopes} that guarantee buyer--side feasibility of the \textsc{TripleWin} fixed point when all knobs enter as \emph{scalings} of fixed baselines. We parameterize
\[
\kappa_{D_i}=\alpha_{\kappa_D}\,\kappa^{(0)}_{D_i},\ 
\kappa_{M_j}=\alpha_{\kappa_M}\,\kappa^{(0)}_{M_j},\ 
\delta_{M_j}=\alpha_{\delta}\,\delta^{(0)}_{M_j},
\]
with $\alpha_{\kappa_D}>0,\ \alpha_{\kappa_M}>0,\ \alpha_\delta\ge 0$, and define for each model $M_j$ the baseline data--offset sum
\[
S_j \;=\; \sum_{i\in \mathcal{D}_{M_j}}\kappa^{(0)}_{D_i}\,.
\]
At any \textsc{TripleWin} fixed point the quoted and posted prices coincide, and for each model $M_j$,
\begin{multline}
\sum_{i\in \mathcal{D}_{M_j}} p_{D_i\to M_j} 
= \sum_{i\in \mathcal{D}_{M_j}}\Big(\kappa_{D_i}+\frac{\mathrm{SV}_{i\mid j}}{1+\delta_{M_j}}\,W_j\Big) \\
= \alpha_{\kappa_D}S_j + \frac{W_j}{1+\alpha_\delta\,\delta^{(0)}_{M_j}}, 
\label{eq:sum-D-edge-scale}
\end{multline}
and
\begin{multline}
p_{B_k\to M_j} 
= \kappa_{M_j} + (1+\delta_{M_j})\sum_{i\in \mathcal{D}_{M_j}} p_{D_i\to M_j} \\
= \alpha_{\kappa_M}\kappa^{(0)}_{M_j} + \bigl(1+\alpha_\delta\,\delta^{(0)}_{M_j}\bigr)\alpha_{\kappa_D}S_j + W_j. 
\label{eq:p-B-equality-scale}
\end{multline}
Buyer feasibility is $p_{B_k\to M_j}\le R_{B_k\to M_j}$ for all $k\in\mathcal{B}_{M_j}$. We write
\[
\bar R_j \;=\; \max_{k\in\mathcal{B}_{M_j}} R_{B_k\to M_j},\qquad
\underline R_j \;=\; \min_{k\in\mathcal{B}_{M_j}} R_{B_k\to M_j},
\]
and use the elementary bound
\begin{multline}
0 \;\le\; W_j \;=\; \sum_k\omega_{jk}p_{B_k\to M_j} \\ 
\;\le\;\sum_k\omega_{jk}R_{B_k\to M_j}\;\le\;\rho_j\,\bar R_j.
\label{eq:Qj-upper}
\end{multline}

The next lemmas give sufficient feasibility conditions expressed \emph{only} in the three scalings $(\alpha_{\kappa_D},\alpha_{\kappa_M},\alpha_\delta)$.

\begin{lemma}[Envelope $\alpha_\delta$ vs.\ $\alpha_{\kappa_D}$]\label{lem:alpha-delta-vs-alphaD}
Fix $\alpha_{\kappa_M}>0$. For any model $M_j$ and buyer $B_k\in\mathcal{B}_{M_j}$, if
\begin{equation}
\bigl(1+\alpha_\delta\,\delta^{(0)}_{M_j}\bigr)\,\alpha_{\kappa_D}\,S_j \;\le\; R_{B_k\to M_j} - \alpha_{\kappa_M}\kappa^{(0)}_{M_j} - \rho_j\,\bar R_j,
\label{eq:alpha-delta-alphaD-ineq}
\end{equation}
then $p_{B_k\to M_j}\le R_{B_k\to M_j}$ at the fixed point. Equivalently, when $\delta^{(0)}_{M_j}>0$,
\[
\alpha_\delta\ \le\ \min_{k\in\mathcal{B}_{M_j}}
\left(\frac{R_{B_k\to M_j}-\alpha_{\kappa_M}\kappa^{(0)}_{M_j}-\rho_j\,\bar R_j}{\alpha_{\kappa_D}S_j}-1\right)\!\Big/\delta^{(0)}_{M_j},
\]
while if $\delta^{(0)}_{M_j}=0$ the constraint does not restrict $\alpha_\delta$. A market--wide envelope is
\begin{multline*}
\alpha_\delta^{\max}(\alpha_{\kappa_D}) \\
=\!\!\!\!\!\min_{j\in[J],\,k\in\mathcal{B}_{M_j}:\,\delta^{(0)}_{M_j}>0} \!\!\!
\left(\frac{R_{B_k\to M_j}-\alpha_{\kappa_M}\kappa^{(0)}_{M_j}-\rho_j\,\bar R_j}{\alpha_{\kappa_D}S_j}-1\right)\!\Big/\delta^{(0)}_{M_j}.
\end{multline*}
\end{lemma}

The term $\bigl(1+\alpha_\delta\,\delta^{(0)}_{M_j}\bigr)\alpha_{\kappa_D}S_j$ is the (scaled) data expenditure in buyer prices. Inequality~\eqref{eq:alpha-delta-alphaD-ineq} guarantees that, even under worst--case aggregation $W_j=\rho_j\bar R_j$, the reserve still covers the model offset and marked--up data spend. The envelope is monotone decreasing in $\alpha_{\kappa_D}$ and~$\rho_j$.

\begin{proof}
From \eqref{eq:p-B-equality-scale} and $W_j\le\rho_j\bar R_j$,
\[
p_{B_k\to M_j} \;\le\; \alpha_{\kappa_M}\kappa^{(0)}_{M_j} + \bigl(1+\alpha_\delta\,\delta^{(0)}_{M_j}\bigr)\alpha_{\kappa_D}S_j + \rho_j\,\bar R_j,
\]
so $p_{B_k\to M_j}\le R_{B_k\to M_j}$ is ensured by \eqref{eq:alpha-delta-alphaD-ineq}. Solving for $\alpha_\delta$ gives the stated bound; minimizing over $(j,k)$ yields the market envelope.
\end{proof}

\begin{lemma}[Envelope $\alpha_\delta$ vs.\ $\alpha_{\kappa_M}$]\label{lem:alpha-delta-vs-alphaM}
Fix $\alpha_{\kappa_D}>0$. For any model $M_j$ and buyer $B_k\in\mathcal{B}_{M_j}$, the condition
\[
\bigl(1+\alpha_\delta\,\delta^{(0)}_{M_j}\bigr)\,\alpha_{\kappa_D}\,S_j \;\le\; R_{B_k\to M_j} - \alpha_{\kappa_M}\kappa^{(0)}_{M_j} - \rho_j\,\bar R_j
\]
again suffices for $p_{B_k\to M_j}\le R_{B_k\to M_j}$. Equivalently,
\[
\alpha_\delta \le\!\!\!\min_{k\in\mathcal{B}_{M_j}:\,\delta^{(0)}_{M_j}>0}
\left(\frac{R_{B_k\to M_j}-\alpha_{\kappa_M}\kappa^{(0)}_{M_j}-\rho_j\,\bar R_j}{\alpha_{\kappa_D}S_j}-1\right)\!\Big/\delta^{(0)}_{M_j},
\]
with the same interpretation for $\delta^{(0)}_{M_j}=0$. The market envelope $\alpha_\delta^{\max}(\alpha_{\kappa_M})$ is obtained by minimizing over $(j,k)$.
\end{lemma}

\begin{lemma}[Envelope $\alpha_{\kappa_M}^{\max}$ and $\alpha_{\kappa_D}^{\max}$ at fixed $\alpha_\delta$]\label{lem:kM-kD-vs-alpha-delta}
For fixed $\alpha_\delta\ge 0$ and any $(j,k)$, sufficient conditions for $p_{B_k\to M_j}\le R_{B_k\to M_j}$ are
\[
\alpha_{\kappa_M}\,\kappa^{(0)}_{M_j} \;\le\; R_{B_k\to M_j} - \bigl(1+\alpha_\delta\,\delta^{(0)}_{M_j}\bigr)\alpha_{\kappa_D}S_j - \rho_j\,\bar R_j,
\]
and
\[
\alpha_{\kappa_D}\,S_j \;\le\; \frac{R_{B_k\to M_j}-\alpha_{\kappa_M}\kappa^{(0)}_{M_j}-\rho_j\,\bar R_j}{1+\alpha_\delta\,\delta^{(0)}_{M_j}}.
\]
Consequently,
\begin{multline*}
\alpha_{\kappa_M}^{\max}(\alpha_{\kappa_D};\alpha_\delta) \\
=\min_{j\in[J],\,k\in\mathcal{B}_{M_j}}\frac{R_{B_k\to M_j}-\bigl(1+\alpha_\delta\,\delta^{(0)}_{M_j}\bigr)\alpha_{\kappa_D}S_j-\rho_j\,\bar R_j}{\kappa^{(0)}_{M_j}},
\end{multline*}
and
\[
\alpha_{\kappa_D}^{\max}(\alpha_{\kappa_M};\alpha_\delta)
=\min_{j\in[J],\,k\in\mathcal{B}_{M_j}}\frac{R_{B_k\to M_j}-\alpha_{\kappa_M}\kappa^{(0)}_{M_j}-\rho_j\,\bar R_j}{\bigl(1+\alpha_\delta\,\delta^{(0)}_{M_j}\bigr)S_j}.
\]
\end{lemma}

\begin{theorem}[Global feasibility region and comparative statics]\label{thm:global-feasibility-alpha}
Consider any instance with the scalings above, buyer weights $\sum_k\omega_{jk}=\rho_j\in[0,1)$, and Shapley columns $\sum_{i\in \mathcal{D}_{M_j}}\mathrm{SV}_{i\mid j}=1$. Let $\mathbf p^\star$ denote the unique \textsc{TripleWin} fixed point. If, for all $j\in[J]$ and $k\in\mathcal{B}_{M_j}$,
\begin{equation}
\bigl(1+\alpha_\delta\,\delta^{(0)}_{M_j}\bigr)\,\alpha_{\kappa_D}S_j \;\le\; R_{B_k\to M_j}-\alpha_{\kappa_M}\kappa^{(0)}_{M_j}-\rho_j\,\bar R_j,
\label{eq:global-feasibility-condition-alpha}
\end{equation}
then $\mathbf p^\star$ is buyer--feasible: $p_{B_k\to M_j}^\star\le R_{B_k\to M_j}$ for all $(j,k)$. Moreover, the feasible set in $(\alpha_{\kappa_D},\alpha_{\kappa_M},\alpha_\delta)$ is \emph{downward--closed} (orthant--monotone): if a triple is feasible, any triple with componentwise smaller scalings is also feasible. In single--parameter slices (holding two scalings fixed) the frontier is \emph{piecewise affine} in that parameter; in the $(\alpha_{\kappa_D},\alpha_\delta)$ plane it traces decreasing hyperbolas of the form $\alpha_{\kappa_D}(1+\alpha_\delta\,\delta^{(0)}_{M_j})=\text{const}$ and remains downward--closed.
\end{theorem}
\begin{proof}
From \eqref{eq:p-B-equality-scale} and \eqref{eq:Qj-upper},
\[
p_{B_k\to M_j}^\star \;\le\; \alpha_{\kappa_M}\kappa^{(0)}_{M_j}
+\bigl(1+\alpha_\delta\,\delta^{(0)}_{M_j}\bigr)\alpha_{\kappa_D}S_j + \rho_j\,\bar R_j.
\]
If \eqref{eq:global-feasibility-condition-alpha} holds for all $(j,k)$, then $p_{B_k\to M_j}^\star\le R_{B_k\to M_j}$, proving feasibility. Monotonicity follows because the left--hand side of \eqref{eq:global-feasibility-condition-alpha} is increasing in each of $\alpha_{\kappa_D}$ and $\alpha_\delta$, and the subtracted term $\alpha_{\kappa_M}\kappa^{(0)}_{M_j}$ is increasing in $\alpha_{\kappa_M}$. The geometric claims follow from inspecting \eqref{eq:global-feasibility-condition-alpha}.
\end{proof}

\noindent 
\textbf{Remarks.}
If $\delta^{(0)}_{M_j}=0$ (no baseline margin for model $j$), then $\alpha_\delta$ does not affect the feasibility for that model.  If for some $(j,k)$ the right--hand side of \eqref{eq:global-feasibility-condition-alpha} is nonpositive, no choice of scalings makes that buyer--model edge feasible; the instance must be relaxed (larger reserves or smaller offsets/weights).  The envelopes in Lemmas~\ref{lem:alpha-delta-vs-alphaD}--\ref{lem:kM-kD-vs-alpha-delta} are sufficient and become tight when the binding buyer for a model is also the one with the highest reserve, i.e., when $\bar R_j=R_{B_k\to M_j}$ at the boundary. In our experiments, numerical frontiers traced by \textsc{TripleWin} closely track these analytic envelopes.

\subsection{Platform Commission as a Uniform Ad–Valorem Fee}
\label{subsec:platform-commission}

We consider a central platform that takes a uniform ad–valorem fee \(\tau\in[0,1)\) from all monetary transfers on both layers of the market. A buyer payment \(p_{B_k\to M_j}\) delivers \((1-\tau)\,p_{B_k\to M_j}\) to the producer, and a producer payment \(p_{D_i\to M_j}\) delivers \((1-\tau)\,p_{D_i\to M_j}\) to the data seller. Let \(\alpha = 1/(1-\tau)\ge 1\) denote the induced grossing factor. With this fee the buyer-side quotation must gross up the model-side offset \(\kappa_{M_j}\), while the pass-through of data expenditure to buyers remains untaxed at the quoting stage because the producer already pays data sellers the gross amount. Symmetrically, the data-side quotation must gross up the dataset-side offset \(\kappa_{D_i}\). The correct quotations under a uniform platform fee are therefore
\begin{align}
v_{B_k\to M_j} \;&=\; \alpha\,\kappa_{M_j} + (1+\delta_{M_j}) \sum_{i\in \mathcal{D}_{M_j}} p_{D_i\to M_j}, 
\label{eq:fee-buyer-quote}\\
v_{D_i\to M_j} \;&=\; \alpha\,\kappa_{D_i} + \frac{\mathrm{SV}_{i\mid j}}{1+\delta_{M_j}}\,W_j, 
\label{eq:fee-data-quote}
\end{align}
with $W_j=\sum_{k\in\mathcal{B}_{M_j}}\omega_{jk}\,p_{B_k\to M_j}$ and \(\sum_{k\in\mathcal{B}_{M_j}}\omega_{jk}=\rho_j\in[0,1)\).

\begin{lemma}[Buyer price at the fixed point under a uniform fee]
\label{lem:buyer-price-fee}
Fix a model \(M_j\) and define \(S_j=\sum_{i\in \mathcal{D}_{M_j}}\kappa_{D_i}\). At any \textsc{TripleWin} fixed point under the uniform fee \(\tau\) with \(\alpha=1/(1-\tau)\), the buyer price is independent of \(k\) and satisfies
\begin{equation}
p_{B_k\to M_j} \;=\; p_{M_j} 
\;=\; \frac{\alpha\,\big(\kappa_{M_j} + (1+\delta_{M_j})S_j\big)}{1-\rho_j},
\ k\in\mathcal{B}_{M_j}.
\label{eq:pmj-fee}
\end{equation}
\end{lemma}

\begin{proof}
At a fixed point we have \(p_{D_i\to M_j}=v_{D_i\to M_j}\) and \(p_{B_k\to M_j}=v_{B_k\to M_j}\). Summing the data-side prices over \(i\in \mathcal{D}_{M_j}\) and using \(\sum_{i\in \mathcal{D}_{M_j}}\mathrm{SV}_{i\mid j}=1\) gives
\begin{multline}
\sum_{i\in \mathcal{D}_{M_j}} p_{D_i\to M_j}
\;=\; \sum_{i\in \mathcal{D}_{M_j}}\left(\alpha\,\kappa_{D_i} + \frac{\mathrm{SV}_{i\mid j}}{1+\delta_{M_j}}\,W_j\right) \\
\;=\; \alpha\,S_j + \frac{W_j}{1+\delta_{M_j}}.
\label{eq:sum-D-fee}
\end{multline}
Substituting \eqref{eq:sum-D-fee} into the buyer-side quotation \eqref{eq:fee-buyer-quote} yields
\begin{align}
p_{B_k\to M_j}
&= \alpha\,\kappa_{M_j} + (1+\delta_{M_j})\left(\alpha\,S_j + \frac{W_j}{1+\delta_{M_j}}\right) \nonumber\\
&= \alpha\,\kappa_{M_j} + \alpha\,(1+\delta_{M_j})\,S_j + W_j.
\label{eq:pbk-linear}
\end{align}
The right-hand side of \eqref{eq:pbk-linear} does not depend on \(k\), hence all buyers of \(M_j\) face the same price which we denote by \(p_{M_j}\). Aggregating with the buyer weights gives
\begin{equation}
W_j \;=\; \sum_{k\in\mathcal{B}_{M_j}}\omega_{jk}\,p_{B_k\to M_j}
\;=\; \rho_j\,p_{M_j}.
\label{eq:Q-equals-rho-p}
\end{equation}
Combining \eqref{eq:pbk-linear} and \eqref{eq:Q-equals-rho-p} gives
\begin{equation}
p_{M_j} \;=\; \alpha\,\kappa_{M_j} + \alpha\,(1+\delta_{M_j})\,S_j + \rho_j\,p_{M_j}.
\label{eq:pmj-linear-eq}
\end{equation}
Rearranging \eqref{eq:pmj-linear-eq} by moving the last term to the left, and noting \(1-\rho_j>0\), we obtain
\begin{equation}
(1-\rho_j)\,p_{M_j} \;=\; \alpha\,\big(\kappa_{M_j} + (1+\delta_{M_j})\,S_j\big).
\end{equation}
Dividing both sides by \(1-\rho_j\) proves \eqref{eq:pmj-fee}.
\end{proof}

The buyer-side feasibility constraint for model \(j\) requires \(p_{M_j}\le R_{j}^{\min}\), where \(R_{j}^{\min}=\min_{k\in\mathcal{B}_{M_j}} R_{B_k\to M_j}\). Substituting \eqref{eq:pmj-fee} and multiplying both sides by \(1-\rho_j>0\) gives the linear inequality
\begin{equation}
\alpha\,\big(\kappa_{M_j} + (1+\delta_{M_j})\,S_j\big) \;\le\; (1-\rho_j)\,R_{j}^{\min}.
\label{eq:alpha-ineq-linear}
\end{equation}
Solving \eqref{eq:alpha-ineq-linear} for \(\alpha\) yields the model-wise upper bound
\begin{equation}
\alpha \;\le\; \alpha_j^{\max}
\;=\; \frac{(1-\rho_j)\,R_{j}^{\min}}{\kappa_{M_j} + (1+\delta_{M_j})\,S_j}\,,
\quad \forall j
\label{eq:alpha-max-j}
\end{equation}
The platform can charge a uniform fee only if \(\alpha\) satisfies \eqref{eq:alpha-max-j} for all \(j\).

\begin{theorem}[Maximal uniform platform fee]
\label{thm:max-uniform-fee}
Assume \(\kappa_{M_j}> 0\), \(S_j> 0\), \(\delta_{M_j}> 0\), \(R_{j}^{\min}>0\), and \(\rho_j\in[0,1)\) for all \(j\). Define
\begin{equation}
\alpha^\star = \min_{1\le j\le J} \alpha_j^{\max}
= \min_{1\le j\le J} \frac{(1-\rho_j)\,R_{j}^{\min}}{\kappa_{M_j} + (1+\delta_{M_j})\,S_j}\,.
\label{eq:alpha-star}
\end{equation}
Then \(\alpha^\star\) is the largest grossing factor for which there exists a buyer-feasible \textsc{TripleWin} fixed point under the uniform fee, and the corresponding maximal fee is
\begin{equation}
\tau^\star \;=\; \max\Big\{\,0,\;1-\frac{1}{\alpha^\star}\,\Big\}.
\label{eq:tau-star}
\end{equation}
\end{theorem}

\begin{proof}
We first show that any \(\alpha\) with \(\alpha\le \alpha^\star\) is buyer-feasible. For each \(j\), the definition \eqref{eq:alpha-star} implies \(\alpha\le \alpha_j^{\max}\). Substituting \(\alpha\le \alpha_j^{\max}\) into \eqref{eq:alpha-ineq-linear} gives
\begin{equation}
\alpha\,\big(\kappa_{M_j} + (1+\delta_{M_j})\,S_j\big) \;\le\; (1-\rho_j)\,R_{j}^{\min}.
\end{equation}
Dividing both sides by \(1-\rho_j>0\) yields
\begin{equation}
p_{M_j}(\alpha) \;=\; \frac{\alpha\,\big(\kappa_{M_j} + (1+\delta_{M_j})\,S_j\big)}{1-\rho_j}
\;\le\; R_{j}^{\min}.
\end{equation}
This proves buyer feasibility for all \(j\).

We next show that no \(\alpha>\alpha^\star\) is buyer-feasible. By the definition \eqref{eq:alpha-star}, there exists at least one index \(j^\star\) such that \(\alpha^\star=\alpha_{j^\star}^{\max}\) and for every \(\epsilon>0\), \(\alpha^\star+\epsilon>\alpha_{j^\star}^{\max}\). Taking any \(\epsilon>0\) and substituting \(\alpha=\alpha^\star+\epsilon\) into \eqref{eq:alpha-ineq-linear} with \(j=j^\star\) gives
\begin{equation}
(\alpha^\star+\epsilon)\,\big(\kappa_{M_{j^\star}} + (1+\delta_{M_{j^\star}})\,S_{j^\star}\big) \;>\; (1-\rho_{j^\star})\,R_{j^\star}^{\min}.
\end{equation}
Dividing by \(1-\rho_{j^\star}>0\) yields \(p_{M_{j^\star}}(\alpha^\star+\epsilon) > R_{j^\star}^{\min}\), hence buyer feasibility fails for model \(j^\star\). This proves maximality of \(\alpha^\star\).

Finally, \(\tau=1-1/\alpha\) is a strictly increasing transformation of \(\alpha\) on \([1,\infty)\). If \(\alpha^\star<1\) the nonnegativity constraint \(\tau\ge 0\) binds and the maximal admissible fee is \(\tau^\star=0\); otherwise \(\tau^\star=1-1/\alpha^\star\). This establishes \eqref{eq:tau-star}.
\end{proof}

The dependence of equilibrium prices on the platform fee is monotone and linear. Differentiating \eqref{eq:pmj-fee} with respect to \(\alpha\) gives
\begin{equation}
\frac{\partial p_{M_j}}{\partial \alpha}
\;=\; \frac{\kappa_{M_j} + (1+\delta_{M_j})\,S_j}{1-\rho_j}
\;\ge\; 0,
\end{equation}
so higher fees uniformly increase buyer prices, with pass-through amplified as \(\rho_j\) approaches one. The feasibility region therefore shrinks linearly in \(\alpha\), and it collapses when offsets and margins are large relative to buyer reserves or when the effective aggregation weight \(\rho_j\) is close to one.

Computation of the maximal fee is closed form. For each model \(j\) compute \(\alpha_j^{\max}\) by \eqref{eq:alpha-max-j} and set \(\alpha^\star=\min_j \alpha_j^{\max}\); the fee is \(\tau^\star\) from \eqref{eq:tau-star}. As a diagnostic one may also verify tightness by evaluating \(p_{M_j}(\alpha^\star)\) for all \(j\). At least one model must satisfy \(p_{M_j}(\alpha^\star)=R_{j}^{\min}\), and nudging \(\alpha\) upward by any \(\epsilon>0\) makes that model violate its reserve.

\section{Experiments}\label{sec:experiments}

% ======================= Experiment Setup (USD) =========================
\begin{table}[t]
\centering
\caption{Simulation configuration used across experiments.}
\label{tab:exp-setup-usd}
\small
\begin{tabular}{@{}ll@{}}
\toprule
Quantity & Value / distribution (USD) \\
\midrule
Buyer weights & $\omega_{jk}\ge 0$, $\sum_{k}\omega_{jk}=\rho_j\in[0,1)$ \\
Dataset offsets (per-use) & $\kappa_{D_i}\sim\mathrm{Unif}[0.10,\,0.40]$ \\
Model offsets (per-sale) & $\kappa_{M_j}\sim\mathrm{Unif}[1.0,\,5.0]$ \\
Producer margins & $\delta_{M_j}=0.10$ for all $j$ (10\% markup) \\
Buyer reserves & $R_{B_k\to M_j}\sim\mathrm{Unif}[25,\,100]$ \\
Buyers per model & $|\mathcal{B}_{M_j}|\sim\mathrm{Unif}\{1,2,3,4,5\}$ \\
List caps & $C_{D_i\to M_j}\sim\mathrm{Unif}[1.5,\,4.0]$ \\
Initialization & $p^{(0)}_{D\to M}=C_{D\to M}$,\quad $p^{(0)}_{B\to M}=\kappa_{M}$ \\
Stopping rule & $\|\mathcal Q(\mathbf p)-\mathbf p\|_2/\sqrt{d}\le 10^{-10}$ \\
\bottomrule
\end{tabular}
\end{table}

\begin{table*}[htbp]
    \centering
    \caption{Performance Comparison of Different Methods}
    \label{tab:performance_comparison_no_cat}
    \resizebox{\textwidth}{!}{%
\begin{tabular}{llllllllll}
\toprule
Method & Social Welfare & Buyer Surplus & Seller Surplus & Platform Profit & E2E Succ. Rate & Avg. Buyer Price & Avg. Seller Price & Spearman ($\rho$) & Noise Resil. \\
\midrule
SupplyFirst & 317.1$\pm$109.8 & \cellcolor{BuyerBlue!10}71.4$\pm$31.0 & \cellcolor{SellerOrange!17}107.4$\pm$20.7 & \cellcolor{PlatformGreen!38}138.3$\pm$73.1 & 47.27\%$\pm$12.28\% & 33.03$\pm$1.13 & 2.77$\pm$0.09 & -0.02$\pm$0.16 & 49.62\%$\pm$5.35\% \\
DemandFirst & \textbf{519.7$\pm$92.5} & \cellcolor{BuyerBlue!39}532.1$\pm$92.7 & \cellcolor{SellerOrange!10}8.3$\pm$0.8 & \cellcolor{NegRed!22}-20.7$\pm$1.7 & 80.86\%$\pm$26.32\% & 3.06$\pm$0.26 & 0.42$\pm$0.03 & 0.45$\pm$0.11 & 70.69\%$\pm$5.38\% \\
BrokerCentric & 345.6$\pm$82.7 & \cellcolor{BuyerBlue!10}71.4$\pm$20.2 & \cellcolor{SellerOrange!18}125.1$\pm$6.2 & \cellcolor{PlatformGreen!39}149.1$\pm$70.7 & 54.05\%$\pm$7.66\% & 32.32$\pm$2.26 & 2.77$\pm$0.09 & -0.00$\pm$0.14 & 47.57\%$\pm$4.83\% \\
VAP & \textbf{521.2$\pm$92.5} & \cellcolor{BuyerBlue!35}457.6$\pm$84.3 & \cellcolor{SellerOrange!10}8.0$\pm$1.4 & \cellcolor{PlatformGreen!25}55.5$\pm$10.0 & 35.54\%$\pm$36.24\% & 7.22$\pm$0.31 & 0.43$\pm$0.02 & -0.46$\pm$0.09 & 48.62\%$\pm$0.28\% \\
CMAB-HS & \textbf{525.9$\pm$92.6} & \cellcolor{BuyerBlue!39}521.7$\pm$91.3 & \cellcolor{SellerOrange!10}1.1$\pm$0.1 & \cellcolor{PlatformGreen!17}3.1$\pm$1.8 & 9.54\%$\pm$14.91\% & 3.63$\pm$0.27 & 0.28$\pm$0.04 & -0.03$\pm$0.13 & 50.47\%$\pm$3.03\% \\
Dealer & \textbf{483.7$\pm$80.1} & \cellcolor{BuyerBlue!12}111.2$\pm$30.6 & \cellcolor{SellerOrange!39}420.6$\pm$71.7 & \cellcolor{NegRed!39}-48.1$\pm$7.3 & 88.53\%$\pm$7.97\% & 27.42$\pm$3.53 & 9.62$\pm$1.37 & 0.65$\pm$0.11 & 61.38\%$\pm$6.14\% \\
TripleWin & \textbf{516.9$\pm$91.5} & \cellcolor{BuyerBlue!26}327.7$\pm$67.5 & \cellcolor{SellerOrange!12}39.4$\pm$3.0 & \cellcolor{PlatformGreen!39}149.9$\pm$29.5 & \textbf{98.26\%$\pm$4.23\%} & 14.46$\pm$1.00 & 1.04$\pm$0.08 & \textbf{0.90$\pm$0.04} & \textbf{92.06\%$\pm$1.89\%} \\
\bottomrule
\end{tabular}
    }
\end{table*}

\begin{figure*}[t]
    \centering
    \includegraphics[width=0.32\textwidth]{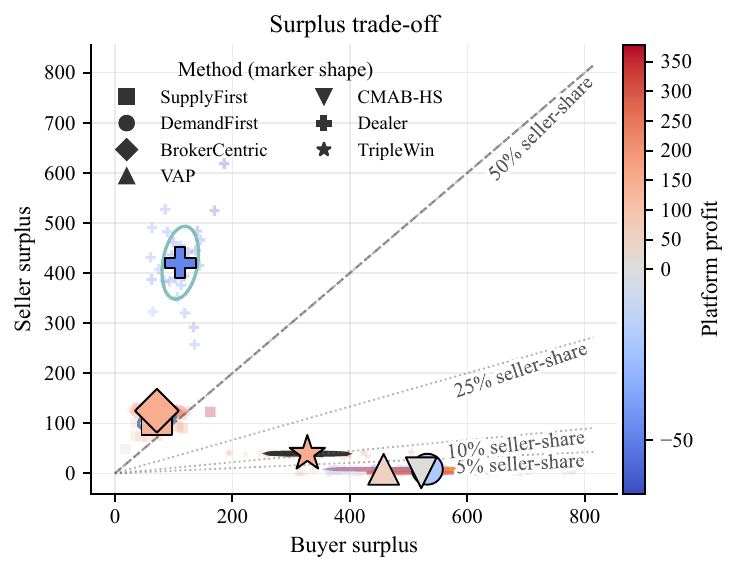} \hfill
    \includegraphics[width=0.32\textwidth]{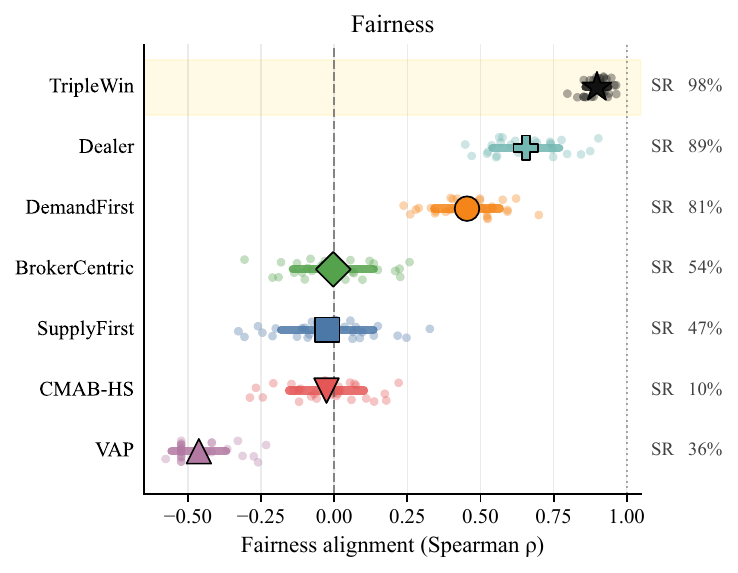} \hfill
    \includegraphics[width=0.32\textwidth]{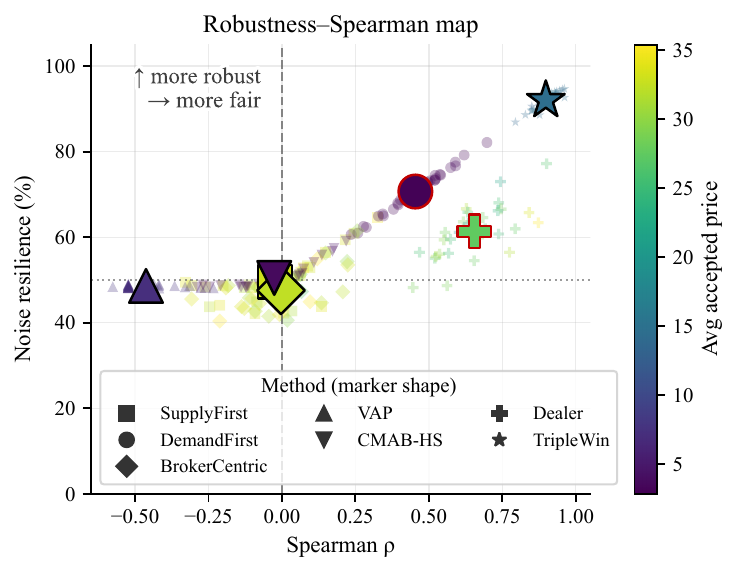}
    \caption{
    Benchmarking pricing mechanisms in a coupled data--model market.
    (left) \emph{Surplus incidence map}: buyer surplus versus seller surplus across random market instances; color encodes platform profit. Grey rays indicate iso seller-share contours \(S/(B+S)\) (10\%, 25\%, 50\%). 
    (middle) \emph{Fairness alignment}: Spearman \( \rho \) between Shapley contributions and realized data-revenue shares; annotations report transaction success rate (SR). 
    (right) \emph{Robustness--fairness map}: noise resilience (\%) under a downward demand shock versus fairness, colored by average accepted price. 
    Across all views, \textsc{TripleWin} attains the most balanced surplus split while delivering the strongest fairness and robustness with near-unit transaction success.
    }
    \label{fig:baseline_comparison}
\end{figure*}

\begin{figure*}[t]
    \centering
    \includegraphics[width=\linewidth]{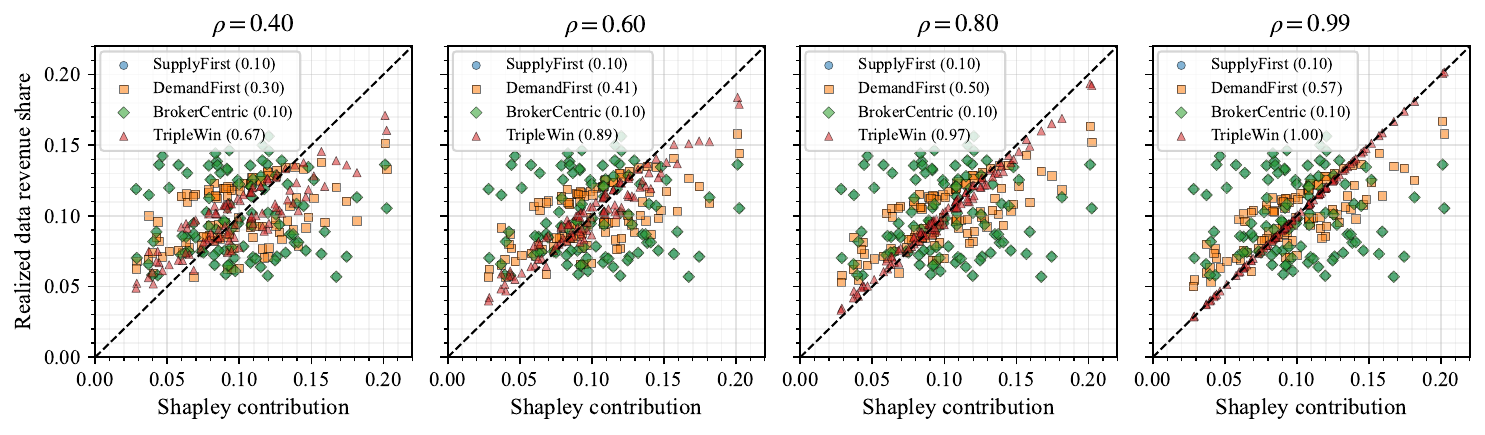}
\caption{Fairness alignment between Shapley contributions and realized data–revenue shares
under different total buyer weights~$\rho\in\{0.4,0.6,0.8,0.99\}$.
Each panel plots the realized revenue share $p_{D_i\to M_j}/\sum_i p_{D_i\to M_j}$
against the corresponding Shapley value $\mathrm{SV}_{i\mid j}$ for all datasets and models.
Legends report mean per–model Spearman correlations.
\textsc{TripleWin} (red triangles) aligns almost perfectly with Shapley contributions
as $\rho$ increases, while one–sided baselines remain weakly correlated.}
    \label{fig:fairness}
\end{figure*}

\begin{figure*}[t]
    \centering
    \includegraphics[width=\linewidth]{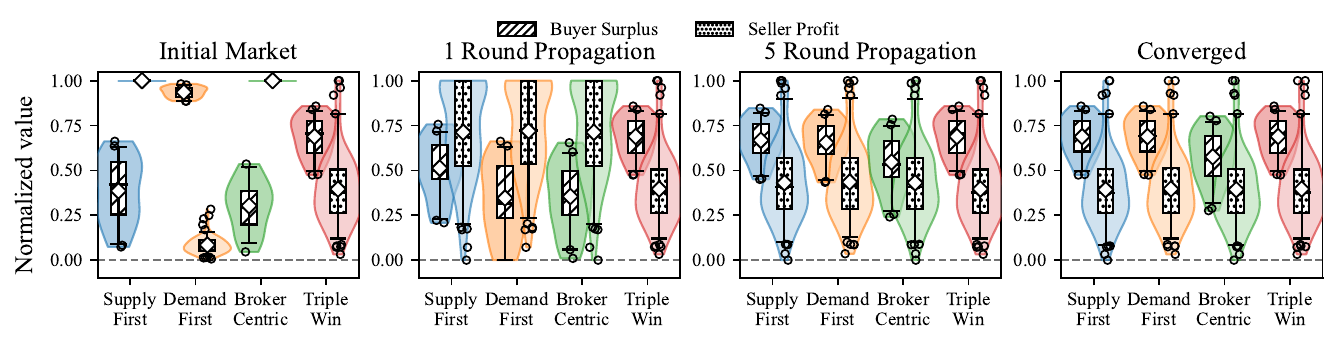}
\caption{Evolution of buyer surplus and seller profit across propagation stages.
Each subpanel shows normalized distributions for the four methods at four stages:
initial market, one market price propagation, five market price propagation, and convergence.
Grey boxes correspond to buyer surplus and hatched boxes to seller profit.
Only \textsc{TripleWin} achieves balanced and stable incidence for both sides at convergence,
consistent with a unique tri–sided fixed point.}
    \label{fig:propagation}
\end{figure*}

We evaluate TripleWin in synthetic coupled markets with multiple datasets, models, and buyers.
Each dataset corresponds to a data seller who posts a per-use license price $p_{D_i\to M_j}$ for model $M_j$.
Each model producer sells the trained model to multiple buyers at buyer-specific prices $p_{B_k\to M_j}$.
Given current posted prices, the market computes \emph{buyer-side quotations} and \emph{data-side quotations} using
Eqs.~\eqref{eq:buyer_quote}--\eqref{eq:data_quote}, and TripleWin iterates these coupled quotation updates until convergence (Algorithm~\ref{alg:triplewin}).
For every experiment we draw market primitives from the common configuration in Table~\ref{tab:exp-setup-usd},
run each pricing mechanism, and then compute downstream and upstream acceptance as in Section~\ref{subsec:pricing_mappings}.

To ground data valuation, we compute dataset-to-model approximate Shapley shares on eight canonical ML datasets spanning vision (\texttt{cifar}, \texttt{digits}), medicine (\texttt{breast\_cancer}), social economics (\texttt{adult}), survival analysis (\texttt{titanic}), biology (\texttt{iris}), food science (\texttt{wine}), and graphs (\texttt{cora}). For each task, we train a library of learners: we deploy the high-capacity ResNet specifically for vision-centric tasks, complemented by seven standard models across all datasets (linear/logistic regression, SVM, random forest, decision tree, gradient boosting, kNN, and MLP). We use predictive performance as the utility in the Shapley definition. Given the computational complexity, these approximate Shapley shares are used to capture the relative marginal contribution of each dataset. For each model $M_j$, we normalize the resulting Shapley column so that $\sum_{i\in D_{M_j}} \mathrm{SV}_{i|j}=1$, ensuring that these approximate shares are comparable across tasks and are interpretable as fair revenue-splitting weights.

% To ground data valuation, we compute dataset-to-model Shapley shares on seven canonical ML datasets spanning
% biology (\texttt{iris}), history (\texttt{survival}), graphs (\texttt{citeseer}, \texttt{cora}), medicine (\texttt{breast\_cancer}),
% vision (\texttt{cifar}), and food science (\texttt{wine}).
% For each task, we train seven standard learners (linear/logistic regression, SVM, random forest, decision tree,
% gradient boosting, kNN, and ResNet) and use predictive performance as the utility in the Shapley definition.
% For each model $M_j$, we normalize the Shapley column so that $\sum_{i\in D_{M_j}} \mathrm{SV}_{i|j}=1$,
% ensuring that Shapley shares are comparable across tasks and are interpretable as revenue-splitting weights.

We benchmark TripleWin against representative pricing pipelines and learning-based mechanisms:
\begin{itemize}
  \item \textbf{SupplyFirst (SF):} a one-pass \emph{supply-driven} pipeline that sets data-side prices first (high list prices)
  and propagates training cost to buyers once via Eq.~\eqref{eq:buyer_quote}, without closing the loop.
  \item \textbf{DemandFirst (DF):} a one-pass \emph{demand-driven} pipeline that fixes low buyer-side bids first and propagates
  them upstream once via Eq.~\eqref{eq:data_quote}, without downstream reconciliation.
  \item \textbf{BrokerCentric (BC):} a broker-led heuristic that chooses markups from buyer reserves and then allocates
  implied training revenue back to datasets, capturing a platform-centric clearing perspective.
  \item \textbf{VAP\cite{DBLP:journals/tmc/XuZLWC24}:} an online-learning style baseline that adapts data prices via value-aware updates and then quotes model prices downstream.
  \item \textbf{CMAB-HS\cite{DBLP:conf/icde/AnX0X021}:} a combinatorial bandit + Stackelberg-style baseline that explores/selects dataset subsets and sets payments accordingly.
  \item \textbf{Dealer\cite{DBLP:journals/pvldb/LiuLL0PS21}:} a privacy-aware marketplace baseline that prices models via a platform-led menu and compensates datasets with contribution-based sharing.
\end{itemize}
These baselines cover one-sided propagation, broker-centric clearing, online learning, bandit exploration, and privacy-aware marketplaces.
All methods are evaluated under the \emph{same} primitives (Table~\ref{tab:exp-setup-usd}) and the same metric definitions below.

We report the following metrics (Table~\ref{tab:performance_comparison_no_cat}), averaged over 30 random market instances (mean $\pm$ std):
\begin{itemize}
  \item \textbf{End-to-end success rate (E2E SR):} the fraction of buyer--model edges that are simultaneously feasible,
  i.e., buyer acceptance ($p_{B\to M}\le R$) and data-side acceptance ($p_{D\to M}$ meets sellers' acceptance thresholds) hold for deliverable trades.
  \item \textbf{Buyer surplus} $\sum (R-p_{B\to M})$ over accepted buyer edges; \textbf{seller surplus} $\sum (p_{D\to M}-\kappa_D)$ over paid data edges.
  \item \textbf{Platform profit} as model revenue minus serving overheads and data expenditures.
  \item \textbf{Social welfare} as the sum of the three surpluses above.
  \item \textbf{Avg.~buyer price} as the mean accepted $p_{B\to M}$ (per buyer transaction), and \textbf{Avg.~seller price} as the mean accepted $p_{D\to M}$ (per dataset--model license).
  \item \textbf{Fairness alignment (Spearman $\rho$):} per-model rank correlation between Shapley contributions and realized data-revenue shares.
  \item \textbf{Noise resilience:} a robustness score under a downward demand shock (we perturb buyer reserves multiplicatively and measure welfare retention, weighted by fairness; higher is better).
\end{itemize}

Table~\ref{tab:exp-setup-usd} is chosen to mimic two empirical regularities of data/model marketplaces.
First, \emph{data-side payments are micro-payments} while \emph{model-side prices are orders of magnitude larger}.
We therefore draw per-use dataset offsets $\kappa_{D_i}\sim \mathrm{Unif}[0.10,0.40]$ to represent recurring governance/compliance,
access control, and preparation costs per training episode, and draw per-sale model offsets $\kappa_{M_j}\sim \mathrm{Unif}[1,5]$ to represent
deployment and serving overheads amortized per sale.
Second, buyer willingness to pay spans a wide band; we draw reserves $R_{B_k\to M_j}\sim \mathrm{Unif}[25,100]$
so that some instances are tight (non-trivial feasibility) while markets do not collapse to all-accept/all-reject regimes.
We set a uniform producer margin $\delta_{M_j}=0.10$ (10\% markup) as a conservative baseline, and use $\rho_j=0.6$
so that the aggregated demand signal $W_j=\sum_k \omega_{jk}p_{B_k\to M_j}$ remains on the scale of a \emph{single} sale
(avoiding double counting across multiple buyers).
Finally, we initialize data-side prices using sellers' list prices $C_{D\to M}\sim \mathrm{Unif}[1.5,4.0]$ and set initial buyer prices to offsets $\kappa_M$,
then run TripleWin to a tight fixed-point tolerance ($10^{-10}$), ensuring that reported outcomes reflect the equilibrium of the coupled operator.

\begin{figure}
    \centering
    \includegraphics[width=1.0\linewidth]{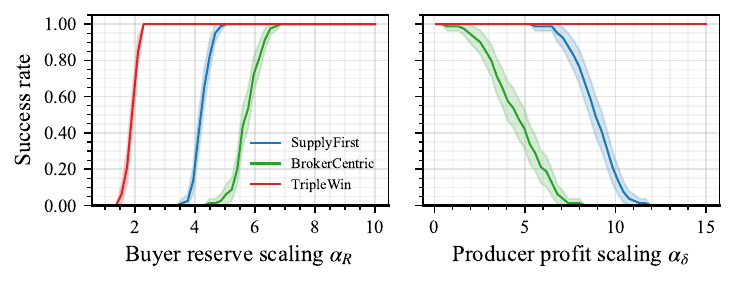}
\caption{Market success rate under symmetric stress tests.
Left: buyer–reserve scaling~$\alpha_R$; right: producer–profit scaling~$\alpha_\delta$.
Curves show the fraction of transactions that remain within acceptance sets
($p_{B_k\to M_j}\!\le\!R_{B_k\to M_j}$ and $p_{D_i\to M_j}\!\ge\!\kappa_{D_i}$)
as the stress parameters vary.
\textsc{TripleWin} (red) sustains near–unit success under tighter reserves
and higher margins, outperforming the SupplyFirst (blue) and BrokerCentric (green) baselines.}
    \label{fig:success}
\end{figure}

\begin{figure}
    \centering
    \includegraphics[width=\linewidth]{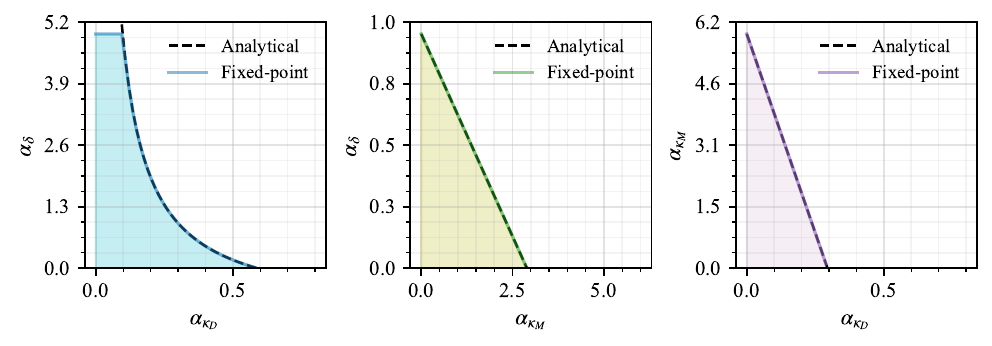}
\caption{Feasible–region envelopes comparing analytic frontiers (dashed)
with numerical \textsc{TripleWin} fixed–point traces (solid).
Panels from left to right show: $(\kappa_D,\delta)$,
$(\kappa_M,\delta)$, and $(\kappa_D,\kappa_M)$ at fixed~$\delta$.
Shaded regions indicate feasible combinations of offsets and margins.
The strong overlap between dashed and solid curves confirms that the numerical equilibria
reach the theoretical feasibility boundary predicted by the envelope formulas.}
    \label{fig:feasibility}
\end{figure}

\subsection{Main benchmark: multi-objective performance and market maps}
\label{sec:exp-main}

\paragraph{Overall comparison (Table~\ref{tab:performance_comparison_no_cat}).}
Table~\ref{tab:performance_comparison_no_cat} summarizes the end-to-end performance of all methods.
TripleWin achieves the \emph{highest} end-to-end transaction success ($98.26\%\pm 4.23\%$) while maintaining strong
platform profit ($149.9\pm 29.5$) and high social welfare ($516.9\pm 91.5$).
Importantly, TripleWin delivers the best fairness alignment (Spearman $\rho=0.90\pm 0.04$), indicating that the realized
data-side revenue shares closely track Shapley contributions.
This alignment translates to superior robustness: TripleWin attains the highest noise resilience ($92.06\%\pm 1.89\%$),
substantially exceeding broker-centric and one-sided pipelines.

In contrast, one-sided pipelines exhibit systematic bias.
DemandFirst pushes prices toward buyers (very low Avg.~buyer price), yielding large buyer surplus but under-compensating sellers
and even negative platform profit, which signals an unsustainable cross-subsidy.
SupplyFirst and BrokerCentric pass high training costs to buyers, resulting in low success rates and near-zero fairness alignment.
Learning-based baselines (VAP and CMAB-HS) attain high buyer surplus on accepted trades but suffer from low end-to-end success,
reflecting that partial/unstable upstream compensation can prevent deliverability in a coupled market.
Dealer strongly favors data suppliers (high seller surplus and high Avg.~seller price) but is less contribution-aligned and less robust than TripleWin. Dealer can report negative platform profit because, consistent with the paper's neutral-broker assumption that passes model-sale revenue through to data owners, the platform retains no margin to cover our serving/operational overheads, yielding a net deficit.

Figure~\ref{fig:baseline_comparison} provides a multi-objective visualization of Table~\ref{tab:performance_comparison_no_cat}.
\emph{Left (surplus incidence):} TripleWin sits near the middle seller-share contours, indicating a balanced incidence between buyer and seller surplus,
while still achieving high platform profit (color).
DemandFirst concentrates nearly all gains on buyers; Dealer shifts gains heavily toward sellers via high data compensation;
SupplyFirst and BrokerCentric remain skewed and clear fewer trades.
\emph{Middle (fairness vs.\ success):} TripleWin is closest to the ideal upper-right region (high $\rho$ and high success rate),
while VAP exhibits negative $\rho$ (misaligned revenue shares) and CMAB-HS suffers very low end-to-end success.
\emph{Right (robustness vs.\ fairness):} under a downward demand shock, TripleWin dominates the robustness--fairness frontier,
showing that closing the bidirectional loop improves both contribution alignment and stability under demand uncertainty.

\subsection{Ablation studies: what drives TripleWin's gains?}
\label{sec:exp-ablation}

For clarity, the following ablations focus on the three canonical competitors that isolate \emph{where the loop is broken}:
SupplyFirst (supply-driven), DemandFirst (demand-driven), and BrokerCentric (platform-led), compared against TripleWin.
Learning- and privacy-based baselines are evaluated in the main benchmark (Table~\ref{tab:performance_comparison_no_cat}, Figure~\ref{fig:baseline_comparison}).

\subsubsection{Ablation A: effect of demand aggregation strength $\rho$}
\label{sec:abl-rho}
Figure~\ref{fig:fairness} varies the total buyer weight $\rho\in\{0.4,0.6,0.8,0.99\}$ and plots, for all dataset--model edges,
the realized data-revenue share $p_{D_i\to M_j}/\sum_i p_{D_i\to M_j}$ against the corresponding Shapley contribution $\mathrm{SV}_{i|j}$.
As $\rho$ increases, TripleWin's point cloud tightens around the $45^\circ$ line and the mean per-model Spearman correlation rises toward one.
This confirms the mechanism's intended behavior: stronger aggregated willingness-to-pay is transmitted upstream and split
\emph{proportionally} to Shapley contributions.
By contrast, one-sided baselines remain weakly correlated because the missing feedback loop attenuates demand signals before they reach data suppliers.

\subsubsection{Ablation B: convergence dynamics of price propagation}
\label{sec:abl-prop}
Figure~\ref{fig:propagation} illustrates how buyer-side surplus and seller-side profit evolve across propagation stages:
(initial market) $\rightarrow$ (one downstream/upstream pass) $\rightarrow$ (five rounds) $\rightarrow$ (convergence).
Remark that each subpanel reports \emph{normalized} values to emphasize the \emph{shape and stability} of incidence,
so absolute levels are intentionally compressed (Table~\ref{tab:performance_comparison_no_cat} and Figure~\ref{fig:baseline_comparison} capture magnitude differences).
Under single-pass mechanisms, incidence remains skewed toward one side and does not stabilize under repeated propagation.
In contrast, TripleWin converges to co-located and stable distributions for both payoffs, which is the empirical footprint of a unique
tri-sided fixed point (guaranteed by the SIF property).
Practically, this means that running the coupled update to convergence yields predictable and repeatable incidence,
rather than depending on the arbitrary ordering of one-shot downstream/upstream passes.

\subsubsection{Ablation C: stress tests on reserves and margins}
\label{sec:abl-stress}
Figure~\ref{fig:success} applies symmetric stress tests by scaling (i) buyer reserves via $\alpha_R$ and (ii) producer margins via $\alpha_\delta$,
and reports the market success rate (fraction of trades inside both acceptance sets).
TripleWin maintains near-unit success under substantially tighter reserves and larger margins than SupplyFirst and BrokerCentric.
Mechanistically, closing the loop distributes shocks across both layers:
tightening reserves contracts demand, which reduces upstream data payments via Shapley allocation and keeps models purchasable;
increasing margins raises buyer prices but also increases the effective training-time signal transmitted to datasets,
helping preserve deliverability.

\subsection{Feasibility envelopes}
\label{sec:exp-envelope}
Finally, Figure~\ref{fig:feasibility} compares the analytic feasibility envelopes derived in Section~IV-D with numerical fixed-point traces.
Across all parameter slices, the numerical frontier closely tracks the analytic boundary, confirming that the fixed-point computation
reaches the theoretical feasibility limit and that the feasible region is downward-closed as predicted.

\section{Conclusion}
\label{sec:conclusion}
We introduced a data–model coupled market in which per–use payments from data sellers to model producers and per–sale bids from buyers to model producers are formed simultaneously through two quotations that close the loop between the supply and demand layers. The resulting operator satisfies the SIF properties, which gives a unique market equilibrium and global convergence of the \textsc{TripleWin} iteration from any nonnegative initialization. The analysis provides transparent comparative statics in model–side and data–side offsets and yields closed-form feasible-region envelopes. These guarantees establish a simple and symmetric clearing mechanism that operates directly at the dataset–model edge and at buyer-specific model prices.

% Empirically, \textsc{TripleWin} aligns realized data revenue with Shapley contributions, balances buyer surplus and seller profit at convergence, and preserves a high transaction success rate under reserve and profit shocks while matching analytical feasibility envelopes. Taken together, the theory and evidence suggest a practical route to tri-sided price formation that is transparent and reproducible. 
Several extensions are natural and left for future work. Strategic behavior and dynamic participation could be modeled explicitly, Shapley shares and demand weights could be learned from noisy signals at scale, and additional constraints such as privacy costs, platform fees, or fairness targets could be incorporated without breaking the fixed-point structure. 
% It is also interesting to endogenize dataset selection, allow heterogeneous or adaptive producer margins, and connect the mechanism to auction primitives so that welfare and budget balance can be evaluated in a unified framework.

% The preferred spelling of the word ``acknowledgment'' in America is without 
% an ``e'' after the ``g''. Avoid the stilted expression ``one of us (R. B. 
% G.) thanks $\ldots$''. Instead, try ``R. B. G. thanks$\ldots$''. Put sponsor 
% acknowledgments in the unnumbered footnote on the first page.

\bibliographystyle{ieeetr}
\bibliography{ref}

\end{document}